\documentclass[twoside,11pt]{article}

\usepackage[preprint]{jmlr2e}

\usepackage{lastpage}
\jmlrheading{26}{2025}{1-\pageref{LastPage}}{1/21; Revised 5/22}{9/22}{21-0000}{Roman Pogodin, Antonin Schrab, Yazhe Li, Danica J. Sutherland, Arthur Gretton. $^*$Work done at Gatsby Computational Neuroscience Unit, University College London and Google DeepMind}

\ShortHeadings{SplitKCI}{Pogodin, Schrab, Li, Sutherland, Gretton}
\firstpageno{1}

\usepackage{our_style}
\newtoggle{inappendix}

\newcommand{\changed}[1]{{#1}}

\usepackage{algorithm2e}
\Crefname{algocf}{Algorithm}{Algorithms}
\SetKwComment{tcp}{\# }{}

\crefname{equation}{Eq.}{Eqs.}
\crefname{figure}{Fig.}{Figs.}

\begin{document}

\title{Practical Kernel Tests of Conditional Independence}

\author{\name Roman Pogodin \email rmn.pogodin@gmail.com \\
       \addr McGill University and Mila
       \AND
       \name Antonin Schrab \email a.schrab@ucl.ac.uk \\
       \addr Gatsby Computational Neuroscience Unit,  University College London\\
       Centre for Artificial Intelligence, University College London
       \AND
       Yazhe Li \email yazheli@outlook.com\\
       \addr Microsoft AI$\,^*$
       \AND 
       Danica J. Sutherland \email dsuth@cs.ubc.ca\\
       \addr University of British Columbia and Amii
       \AND 
       Arthur Gretton \email arthur.gretton@gmail.com \\
       \addr Gatsby Computational Neuroscience Unit, University College London\\
       Google DeepMind
       }

\editor{My editor}

\maketitle

\begin{abstract}%
We describe a data-efficient, kernel-based approach to statistical testing of conditional independence. A major challenge of conditional independence testing is to obtain the correct test level (the specified upper bound on the rate of false positives), while still attaining competitive test power. Excess false positives arise due to bias in the test statistic, which is in our case obtained using nonparametric kernel ridge regression. We propose SplitKCI, an automated method for bias control for the Kernel-based Conditional Independence (KCI) test based on data splitting. We show that our approach significantly improves test level control for KCI without sacrificing test power, both theoretically and for synthetic and real-world data.
\end{abstract}


\section{Introduction}
\label{sec:intro}

Conditional independence (CI) testing is a fundamental problem in statistics: given a confounder $\C$, are two random variables $\A$ and $\B$ dependent?
Applications include basic scientific questions
(``is money associated with self-reported happiness, controlling for socioeconomic status?'');
evaluating the fairness of machine learning methods (the popular \emph{equalized odds} criterion \citep{eq-odds} checks that predictions are independent of a protected status, given the true label);
causal discovery (where conditional independence testing is a core sub-routine of many causal structure discovery methods, building on the PC algorithm \citep{spirtes00,pearl10,sun07kPC,tillman09kPC}),
and more.
If $\C$ takes on a small number of discrete values and sufficient data points are available,
the problem can be essentially reduced to unconditional independence testing for each $\C$ value.
Often, however, $\C$ takes on continuous values and/or has complex structure;
it is then a major challenge to determine how to ``share information'' across similar values of $\C$,
without making strong assumptions about the form of dependence among the variables.

In simple settings,
for instance if $(\A, \B, \C)$ are jointly Gaussian,
conditional independence can be characterized by the
\emph{partial correlation} between $\A$ and $\B$ given $\C$, i.e.
the correlation between residuals of linear regressions $\C \to \A$ and $\C \to \B$.
One family of CI tests extends this insight to nonlinear generalizations of partial correlation,
allowing the detection of nonlinear dependencies,
as suggested by the general characterization of conditional independence of 
\citet{daudin1980partial}.
We consider specifically the Kernel-based Conditional Independence (KCI) \citep{zhang2012kernel} statistic;
and a successor, the Conditional Independence Regression CovariancE (CIRCE) \citep{pogodin2022efficient},
both described in \cref{sec:kernelDep}.
These methods replace the residuals of linear regression with a kernel analogue,
derived from (infinite-dimensional) conditional feature means \citep{SonHuaSmoFuk09,grunewalder2012conditional,li2023optimal};
they then use a kernel-based alternative to covariance
\citep{gretton2005measuring,gretton2007kernel} %
to detect dependence between these residuals.
Given the true conditional feature means and sufficient data for testing,
these methods can detect general modes of conditional dependence, not just linear correlation.

Conditional feature mean estimation from  data is not always a simple task, however: it is difficult in low-data regimes, and for some conditional dependencies it may even be impossible.
In particular, 
nonparametric regression for conditional feature means suffers unavoidably from bias
\citep{li2023optimal,li2024towards}. This bias can easily create dependence between estimated residuals
whose population counterparts are independent, inflating the rate of false positives above the design value for the test.
Bias control is therefore essential in minimizing this effect, and controlling the rate of false positives.

We propose a set of techniques to mitigate this bias,
bringing conditional independence tests closer to their nominal level.
The key to our approach may be understood by comparing KCI with CIRCE, keeping in mind  that both are zero only at conditional independence. The KCI is a covariance of residuals, and thus two regressions are needed in order to compute it (regression from $\C$ to $\A$, and from $\C$ to $\B$). CIRCE, by contrast, requires only regression from  $\C$ to $\B$,  and computes the covariance of this residual with $\A$ alone. In other words, the second regression in KCI, from $\C$ to $\A$, is not necessary for consistency of the test -- both statistics are zero
in population only at conditional independence.
Where there is no true signal (at CI) and the regressions are independent, however, it is natural to expect that {\em residuals} from a regression {\em to} $\A$ may ``appear less dependent'' than $\A$ itself, debiasing the overall test (as formalized in \cref{sec:theory}).
Thus, even imperfect $\C\rightarrow\A$ regression can reduce  bias of the test statistic, and improve the Type I error (false detection rate), without compromising test consistency.

\changed{We introduce SplitKCI, a test statistic that (like KCI) uses both $\C\rightarrow\A$ and $\C\rightarrow\B$ regressions, but builds two independent estimators for one of the regressions. This reduces the effect of the bias from conditional mean estimation on the test statistic, leading to better Type I error control. To also achieve low Type II errors, we introduce a train/test data splitting heuristic that, in practice, improves the ability of the test to maintain its desired Type I error rate without sacrificing much in Type II error. Finally, we prove that the wild bootstrap technique for finding $p$-values for kernel-based test is applicable to KCI variance with empirical conditional mean estimates, improving over the commonly-used gamma approximation.}

\textbf{Testing terminology.}
The null hypothesis (denoted $\hzero$) is $\hzero\!:\! \A\!\indep\!\B\given\C$, that $\A$ is independent of $\B$ given $\C$. The alternative is $\hone\!:\! \A\!\nindep\!\B\given\C$. False positives (rejecting $\hzero$ when it holds) are called Type I errors, and should be bounded by a fixed level $\alpha$, e.g.\ $0.05$. False negatives (failing to reject $\hzero$ under $\hone$) are called Type II errors. Test power is the true positive rate, i.e.\ 1 minus the Type II error.

The rest of the paper is organised as follows. \cref{sec:related_work} discusses related work. \cref{seq:regression_ci} presents the concept of regression-based conditional independence tests. \cref{sec:kernelDep} introduces kernel methods, and in particular kernel methods-based tests. \cref{sec:theory} presents SplitKCI, our approach to conditional independence testing, and discusses the motivation for changes w.r.t.\ the original KCI test. \cref{sec:testing} discusses how  to test for conditional independence with SplitKCI. \cref{sec:experiments} presents experiments on synthetic and real data. \cref{sec:discussion} concludes the paper with a discussion.

\section{Related work}\label{sec:related_work}

Uniformly valid conditional independence tests (with a correctly controlled false positive rate for any distribution) do not exist \citep{shah2020hardness}. \changed{Unlike for for unconditional independence \citep{gretton2007kernel,albert2022adaptive}}, $\A$ and $\B$ may ``seem'' conditionally dependent
until enough data is available to identify that $\C$ actually determines the dependence in a complex way;
for instance,
the construction of \citet{shah2020hardness} essentially ``hides'' information on $\A$ in lower-order bits of $\C$.
Thus a test can never ``commit'' to conditional dependence (or to conditional independence),
in contrast to unconditional settings where tests can successfully ``commit'' to dependence (rejecting the null).
Thus, methods for CI testing must make (often implicit) assumptions on the underlying class of distributions: that is, they work only for a restricted set of nulls and alternatives.

A large body of work uses regression to estimate conditional dependence of $\A$ and $\B$ given $\C$. This is achieved by building the test statistic from the respective residuals of $\C \to \A$ and $\C \to \B$ regressions. Several tests use finite-dimensional residuals, such as the Generalised Covariance Measure (GCM) \citep{shah2020hardness}, weighted GCM \citep{scheidegger2022weighted}, the Projected Covariance Measure (PCM) \citep{lundborg2022projected}, and the Rao-Blackwellized Predictor Test (RBPT) \citep{polo2023conditional}. Kernel-based tests \citep{sun07kPC,fukumizu2007kernel,strobl2019approximate,park2020measure,Huang22condDepend,scetbon2022asymptotic}  instead rely on infinite-dimensional residuals (also see kernelized GCM, \citealp{fernández2022general}). The latter can generally capture more complicated conditional dependencies, leading to better test power. As we will see below, however, they can struggle with correctly identifying conditional independence. In all cases, the type of regression used for each test is a practical choice; one common choice, including in this paper, is kernel ridge regression, due to its theoretical guarantees for conditional expectation estimation \citep{li2023optimal, li2024towards}.

Another class of tests considers local permutations based on clusters of the conditioning variable $\C$ \citep{fukumizu2007kernel,sen2017model,neykov2021minimax}. This is done to simulate the null and compare it to the test data. However, test power in such approaches is sensitive to the marginal distribution of the conditioning variable $\C$ \citep{kim2022local}.
These approaches can be thought of roughly as
using a regression function from $\C$ to $\A$ and/or $\B$ that is piecewise-constant on these clusters in $\C$.

CI tests are typically equipped with an approximate way to compute p-values based on the specific properties of each statistic. In situations where one can (approximately) sample from either the $P(\A\given \C)$ or the $P(\B\given \C)$ distribution, several methods can guarantee level control by sampling from such distributions for any test statistic (e.g. CRT \citep{candes2018panning}, CPT \citep{berrett2020conditional}; see for instance the conditional dependence coefficient of \citet{azadkia2021simple} with a test based on CRT \citep{shi2024azadkia}).
Given the ability to sample from these conditional distributions,
one could also find an arbitrarily accurate regression;
our aim is to improve the performance of tests for situations in which this sampling is not possible,
due to a lack of data relative to the complexity of the conditional dependence.

\section{Introduction to regression-based conditional independence testing}\label{seq:regression_ci}

To define our conditional dependence measures,
we will first need to characterize conditional independence. \changed{Based on the results of \citet{daudin1980partial}, we will define conditional independence through conditional expectations. In practice, conditional expectation can be approximated with regression methods, giving rise to regression-based CI testing methods.}

The definitions below are given for expectations of square-integrable test functions;
considering the case where these functions are 0-1 indicators of events
yields the more familiar version based on factoring probabilities.

\begin{definition}[\citealp{daudin1980partial}] Random variables $\A$ and $\B$ are independent conditioned on $\C$, denoted $\A\indep \B\given \C$, if 
\begin{equation*}
    \EE{g(\A, \C) \, h(\B, \C) \given \C} = \EE{g(\A, \C) \given \C}\,\EE{h(\B, \C) \given \C}
\end{equation*}
    $\C$-almost surely, for all square integrable functions $g\in L^2_{\A\C}$ and $h\in L^2_{\B\C}$.
\end{definition}

\begin{theorem}[\citealp{daudin1980partial}]
    \label{th:daudin}
    Each of the following conditions hold if and only if $\A \indep \B \given \C$,\footnote{The condition \eqref{eq:daudin-one-and-a-half} is not explicitly stated by \citet{daudin1980partial}, but follows from \eqref{eq:daudin-one}; see \citet{pogodin2022efficient}, summarised in \cref{app:sec:kci}.}
    where in each case the test functions $h$ are in $L^2_{BC}$
    (i.e.\ $\EE{ h(B, C)^2 } < \infty$),
    and the test functions $g$ are in $L^2_{AC}$ for \eqref{eq:daudin-both}
    or $L^2_{A}$ for \eqref{eq:daudin-one-and-a-half} and \eqref{eq:daudin-one}.
    \begin{alignat}{3}
        \EE{g(\A, \C) \, h(\B, \C)} = 0
        \quad &\forall
        g,\,h\quad \text{ s.t. } \EE{g(\A, \C) \given \C}=0 \text{ and } \EE{h(\B, \C) \given \C} = 0
        \label{eq:daudin-both}
        ;\\
        \EE{g(\A) \, h(\B, \C)} = 0
        \quad &\forall
        g,\,h\quad \text{ s.t. } \EE{g(\A)\given \C}=0 \text{ and } \EE{h(\B, \C) \given \C} = 0
        \label{eq:daudin-one-and-a-half}
        ;\\
        \EE{g(\A) \, h(\B, \C)} = 0
        \quad
         &\forall
        g,\,h\quad \text{ s.t. } \EE{h(\B, \C) \given \C} = 0
        \label{eq:daudin-one}
    .\end{alignat}
\end{theorem}

In practice, \cref{th:daudin} implies that we can select functions $g$ and $h$ and test whether they correlate in order to determine conditional dependence. More specifically, we can explicitly centre two $L^2$ functions (akin to \cref{eq:daudin-both}) and construct a test statistic on the basis of
\begin{equation}
    \frac{1}{n}\sum_{i=1}^n
    \bigl[ g(\a_i, \c_i) - \EE{g(\A, \C) \given \C=\c_i} \bigr]
    \bigl[h(\b_i, \c_i) - \EE{h(\B, \C) \given \C=\c_i}\bigr]
    \label{eq:generic_regression_test}\,.
\end{equation}

Several tests follow this general approach:
for instance, GCM \citep{shah2020hardness} uses $g(\A,\C)=\A, h(\B,\C)=\B$.
Weighted GCM \citep{scheidegger2022weighted} and the Projected Covariance Measure \citep{lundborg2022projected} extend this to non-linear dependencies for better test power.
KCI \citep{zhang2012kernel}, the focus of our study, uses a different form of dependence, to be discussed shortly.
For any choice of $g$ and $h$, under conditional independence the statistic of \eqref{eq:generic_regression_test} will straightforwardly converge to zero with increasing $n$.
Absent conditional independence, however, the behaviour of \eqref{eq:generic_regression_test} is determined by $g$ and $h$. For instance, consider $\A=\C+\xi$, $\B=\C+\xi^2$ for a standard normal $\xi$. If $g$ and $h$ are linear, the test for conditional independence would amount to testing whether $\xi$ and $\xi^2$ are correlated. As they are not, this test would fail to reject the null hypothesis of conditional independence. The kernel-based methods that we consider here address this issue by considering all possible $g$ and $h$ (in a dense subset of $L^2$),
transforming the scalar-valued \eqref{eq:generic_regression_test} into an infinite-dimensional vector.

The category of test statistics in \eqref{eq:generic_regression_test} also requires us to compute the conditional expectation for the given functions. What happens if we estimate the conditional expectations incorrectly? In general, the test statistic might hallucinate dependence. Denote $\muac = \EE{g(\A, \C) \given \C}$ and $\mubc = \EE{g(\B, \C) \given \C}$, and our estimates of these quantities $\hatmuac,\hatmubc$. Under conditional independence, our estimate of \eqref{eq:generic_regression_test} becomes
\begin{align*}
        \frac{1}{n}\sum_i \sqbrackets{g(\a_i, \c_i) - \hatmuac(\c_i)}&\sqbrackets{h(\b_i, \c_i) - \hatmubc(\c_i)}
       \\&\!\!\!\!\!\!= \frac{1}{n}\sum_i \sqbrackets{g(\a_i, \c_i) - \muac(\c_i)}\sqbrackets{h(\b_i, \c_i) - \mubc(\c_i)}
       \\&+ \frac1n \sum_i \sqbrackets{g(\a_i, \c_i) - \muac(\c_i)} \sqbrackets{\mubc(\c_i) - \hatmubc(\c_i)}
       \\&+ \frac1n \sum_i \sqbrackets{\muac(\c_i) - \hatmuac(\c_i)} \sqbrackets{h(\b_i, \c_i) - \mubc(\c_i)} 
       \\&+ \frac{1}{n}\sum_i \sqbrackets{\muac(\c_i) - \hatmuac(\c_i)}\sqbrackets{\mubc(\c_i) - \hatmubc(\c_i)}\,.
\end{align*}
The first term on the right-hand side converges to zero as $n$ grows, since $\A \indep \B \mid \C$.
The remaining terms depend on the regression errors, however, and may not converge to zero with $n$;
the last term in particular may be large for ``bad'' $g$ and $h$
whose regression errors are correlated.
This problem is exacerbated when looking for $g$ and $h$ with strong dependence under our estimated regressions,
as do the kernel-based measures we discuss next.

If both of our conditional mean estimate remain imperfect as the number of test points $n\rightarrow\infty$, tests with sufficient power will always be able to detect this ``spurious'' dependence. In particular this implies that the higher the power of the conditional independence test, the more sensitive it will be to conditional mean estimation errors. One way to combat such behaviour is by computing the test statistic and the regressions on the same data points (as suggested for GCM by  \citealt{shah2020hardness}). We experimentally observe that this approach leads to inflated Type I errors for the kernel-based tests (although our method in \cref{subsec:ttsplit} helps mitigate this; also see experiments in \cref{sec:experiments}).

\section{Kernel-based measures of conditional dependence}\label{sec:kernelDep}

Kernel methods provide a way to construct infinite-dimensional $g$ and $h$ for \eqref{eq:generic_regression_test} that, if the conditional means are estimated well, guarantee asymptotic power.
A kernel $k : \mathcal{\A} \times \mathcal{\A} \rightarrow \RR$
is a function such that
$k(\a, \a')=\dotprod{\phi(\a)}{\phi(\a')}_{\mathcal{H}_{\a}}$,
where $\phi \!:\! \mathcal\A \!\to\! \mathcal H_\a$ is called a \emph{feature map}
and $\mathcal H_\a$ is a reproducing kernel Hilbert space (RKHS) of functions $f :  \mathcal{\A} \rightarrow  \RR$.
The reproducing property of $\mathcal{H}_{\a}$ states that $\dotprod{\phi(\a)}{f}_{\mathcal{H}_{\a}}\!=\!f(\a)$ for any $f
\!\in\!\mathcal{H}_{\a}$.
For separable RKHSs $\mathcal{H}_{\a}$ and $\mathcal{H}_{\b}$, we can
define a space \citep[see e.g.][]{gretton2013introduction} of Hilbert-Schmidt operators  $A,B :  \mathcal{H}_{\a} \rightarrow \mathcal{H}_{\b}$, denoted $\mathrm{HS}(\mathcal{H}_{\a}, \mathcal{H}_{\b})$, with the inner product
\changed{
\begin{align*}
    \dotprod{A}{B}_{\mathrm{HS}(\mathcal{H}_{\a}, \mathcal{H}_{\b})} = \sum_{i\in I,\,j\in J} \dotprod{Ag_i}{h_j}_{\mathcal{H}_{\b}}\dotprod{Bg_i}{h_j}_{\mathcal{H}_{\b}}\,,
\end{align*}
where $\{g_i\}_{i\in I}$ is any orthonormal basis of $\mathcal{H}_{\a}$, and $\{h_j\}_{j\in J}$ is any orthonormal basis of $\mathcal{H}_{\b}$.
}

For any $f_\a \in \mathcal H_\a$, $f_\b \in \mathcal H_\b$,
the outer product $f_\b \otimes f_\a \in \mathrm{HS}(\mathcal{H}_{\a}, \mathcal{H}_{\b})$ is Hilbert-Schmidt;
analogously to a Euclidean vector outer product,
this outer product has that
for $g_\a \in \mathcal{H}_{\a}$,
$(f_{\b} \otimes f_{\a}) g_\a = \dotprod{f_{\a}}{g_\a}_{\mathcal{H}_{\a}} f_{\b}$
and $\dotprod{A}{f_{\b}\otimes f_{\a}}_{\mathrm{HS}}=\dotprod{f_{\b}}{A f_{\a}}_{\mathcal{H}_{\b}}$.
We will often shorten $\mathrm{HS}(\mathcal{H}_{\a}, \mathcal{H}_{\b})$ to $\mathrm{HS}$ or $\mathrm{HS}_{\a\b}$.

We can now define two conditional covariance operators \citep{fukumizu2004dimensionality},
based on the conditional mean embedding $\mu_{\A\given \C}(\C)=\EE{\phi_{\a}(\A)\given \C}$.
First, the Kernel-based Conditional Independence (KCI) operator, due to \citet{zhang2012kernel}, is
\begin{equation}
\begin{split}
    \mathfrak C_{\mathrm{KCI}} =&\, \expect\left[
        (\phi_{\a}(\A)-\mu_{\A\given \C}(\C))
        \otimes
        \phi_{\c}(\C)
        \otimes
        (\phi_{\b}(\B) - \mu_{\B\given \C}(\C))
    \right].
\end{split}
\end{equation}
This is a \emph{covariance operator} because, for any $g\in\mathcal{H}_{\a}$ and $h\in\mathcal{H}_{\b\c}$ 
(a space of functions on $\mathcal{\B}\times\mathcal{\C}$),
\begin{align*}
    \dotprod{g}{\mathfrak C_{\mathrm{KCI}}\,h}=\EE{\brackets{g(\A)-\expect[g(\A)|\C]}\brackets{h(\B,\C)-\expect[h(\B,\C)|\C]}}.
\end{align*}
Thus, if (and only if) $\mathfrak C_{\mathrm{KCI}} = 0$,
a version of \eqref{eq:daudin-one-and-a-half} where we additionally restrict $g \in \cH_{\a}$, $h \in \cH_{\b\c}$ holds.
Note that the original definition used $\phi_{\b\c}(\B,\C)\!-\!\mu_{\B\C| \C}(\C)$; if $\phi_{\b\c}(\B,\C) = \phi_{\b}(\B) \otimes \phi_c(\C)$ as in radial basis kernels, the definition here (adapted from \cite{pogodin2022efficient}) is the same but avoids estimation of $\mu_{\C| \C}(\C)=\phi_{\c}(\C)$, which is problematic as discussed in \cref{app:sec:kci}. The original paper also swapped the roles of $\A$ and $\B$.

Although \cref{th:daudin} uses $L^2$ functions,
we can define analogous RKHS statistics using an
\emph{$L^2$-universal} kernel,
one whose RKHS is dense in $L^2$;
common choices such as the Gaussian kernel 
$k(\a, \a')=\exp(-\|\a-\a'\|^2/(2\sigma^2))$
are $L^2$-universal
\citep{fukumizu2007kernel,SriFukLan11}.
By continuity,
if no RKHS functions have nonzero conditional covariance,
then the same applies to $L^2$ functions,
and so 
$\mathfrak C_{\mathrm{KCI}} = 0$ iff $\A \indep \B \given \C$.\footnote{\changed{The proof for KCI and KCI-like statistics can be straightforwardly adapted from Th. 2.5 in \citep{pogodin2022efficient}.} $L^2$-universality of the kernel is sufficient, but probably not necessary; it is not necessary for unconditional dependence \citep{gretton2015simpler,i-characteristic}.}
\changed{
To test for conditional independence, we can therefore use the squared HS norm of the covariance operator as the test statistic, since under the null,
\begin{align*}
    \hzero:\ \lVert \mathfrak C_{\mathrm{KCI}} \rVert_{HS}^2 = 0\,.
\end{align*}
}
Using the CI condition in \eqref{eq:daudin-one}, \citet{pogodin2022efficient} introduced the Conditional Independence Regression CovariancE (CIRCE) operator to avoid centring $\A$ in cases where $\phi_{\a}(\A)$ is changing,
e.g.\ because $\phi_{\a}(\A)$ incorporates learned neep neural net features as a part of the mapping:
\begin{equation}
    \mathfrak C_{\mathrm{CIRCE}} \!=\! \expect\bigl[
        \phi_{\a}(\A)
        \otimes
        \phi_{\c}(\C)
        \otimes
        (\phi_{\b}(\B) - \mu_{\B | \C}(\C))
    \bigr].
    \label{eq:circe}
\end{equation}

\subsection{Conditional mean embeddings via kernel ridge regression}
The main challenge of using KCI is in estimating the conditional mean embeddings (CME) $\muac$ and/or $\mubc$ from data. The standard way to estimate CMEs is through kernel ridge regression (KRR) \citep{SonHuaSmoFuk09,grunewalder2012conditional}. Using $m$ points, $\lambda > 0$ as the ridge regression parameter, and with some abuse of notation denoting $\Phi_{\C}$ to be an operator from $\RR^m$ to $\mathcal{H}_{\c}$ such that $\Phi_{\C}h = \sum_{i=1}^m h_i \,\phi(\c_i)$ and analogously for $\Phi_{\A}$, the KRR estimate of the CME is
\begin{equation}
    \label{eq:cme_operator}
    \hatmuac = \Phi_{\C}(K_{\C} + \lambda\, m\, I_m)^{-1}\Phi_{\A}\,.
\end{equation}
At any point $\c$, the CME is therefore estimated as
\begin{equation}
    \label{eq:cme}
    \hatmuac(\c) = K_{\c\C}(K_{\C} + \lambda\, m\, I_m)^{-1}\Phi_{\A}\,,
\end{equation}
where the $m$-dimensional vector $K_{\c\C}$ is evaluated as $(K_{\c\C})_i = k_\c(\c, \c_i)$.

Performance of KRR for this task is well-studied, making it a useful tool for theoretical analysis of tests that rely on conditional means. As shown by \citet{li2023optimal}, Th. 2, for RKHS-valued $g$ and $h$, KRR converges as  $O( m^{-\frac{\beta - \gamma}{2(\beta+p)}})$ (which is shown to be minimax optimal). Informally speaking, $\beta\in[1, 2]$ and $p>0$ determine the smoothness of the conditional mean operators, and $\gamma$ determines the norm in which we consider KRR convergence. For the $L^2$ norm, applicable to finite-dimensional $g$ and $h$, we have $\gamma=0$ and thus the best-case convergence rate can be arbitrarily close to $O(m^{-1/2})$. For RKHS-valued $g$ and $h$, however, we need to consider RKHS norm convergence corresponding to $\gamma=1$. As a result, the best-case convergence rate is (arbitrarily close to) $O(m^{-1/4})$. In both cases, $p \to \infty$ makes convergence arbitrarily slow.

\section{SplitKCI}\label{sec:theory}

As discussed above, CME estimation in an RKHS with $m$ data points typically converges as $O(m^{-1/4})$ or slower. In contrast, for a fixed CME, the estimate of the KCI statistic converges as $O(n^{-1/2})$ \citep[Cor. C.6]{pogodin2022efficient}. In this section, we show that CME estimation errors appear as (spurious) dependence in the test statistic. This is an issue: for large enough $n$ (e.g.\ if $n=m$ when using the same data for both CME estimation and testing), the $O(m^{-1/4})$ noise bias in the test statistic would be significantly larger than the $O(n^{-1/2})$ variation expected under the null. Thus, this test would frequently reject the null simply due to CME estimation errors and not true conditional dependence, failing to have any level control.

A straightforward solution to this problem is to ensure that $n \ll m$.
This, however, directly reduces the power of the test to detect true conditional dependencies, as well as spurious ones.
We present two approaches to mitigate this problem that when combined, achieve both level control (low Type I error) and high power (low Type II error). In \cref{subseq:splitkci}, we present SplitKCI and demonstrate that it significantly reduces the effect of CME estimation bias on the KCI statistic. In \cref{subsec:ttsplit}, we present a heuristic for choosing the train/test split, i.e.\ $m$ and $n$, to achieve high power without sacrificing level control. 

\subsection{Reducing CME estimation bias in KCI}\label{subseq:splitkci}

We start by analysing the effect of an incorrect CME,
which introduces an additional bias into our dependence measure even at population level, $n = \infty$.
We refer to the $n$ points used for statistic evaluation as $\datatest$, the $m_{\b\c}$ points used for estimating $\mubc$ as $\datatrain{m_{\b\c}}$, and the $m_{\a\c}$ points used for estimating $\muac$ as $\datatrain{m_{\a\c}}$.
For our theoretical analysis, we assume that the test data is independent of the training data. 

We first define a generic KCI-type operator:
for $\beta: \mathcal{\C} \to \mathcal{H}_{\a}$ and $\tau: \mathcal{\C} \to \mathcal{H}_{\b}$,
\begin{equation}
    \mathfrak C\brackets{\beta,\,\tau} = \expect\left[
        (\phi_{\a}(\A)-\beta(\C))
        \otimes
        \phi_{\c}(\C)
        \otimes
        (\phi_{\b}(\B) - \tau(\C))
    \right]\,.\label{eq:generic_cov_like}
\end{equation}
Then $\mathfrak C_{\mathrm{KCI}} = \mathfrak C\brackets{\muac,\,\mubc}$ and $\mathfrak C_{\mathrm{CIRCE}} = \mathfrak C\brackets{0,\,\mubc}$.
In practice, however, we only have access to operators
based on CME estimates $\hatmuac(\c) = \muac(\c) + \widehat b_{\A\given\C}(\c)$ and an analogous $\hatmubc$.
Thus, for given CME estimators, the underlying operators we estimate in a test are
\begin{equation*}
    \widehat{\mathfrak C}_{\mathrm{KCI}} = \mathfrak C\brackets{\muac+\widehat b_{\A\given\C},\,\mubc+\widehat b_{\B\given\C}}
    ,\quad
    \widehat{\mathfrak C}_{\mathrm{CIRCE}} = \mathfrak C\brackets{0, \mubc + \widehat b_{\B\given\C}}
    \,.
\end{equation*}

Under $\hzero$, we have that $\mathfrak C_{\mathrm{KCI}} = \mathfrak C\brackets{\muac,\,\mubc}=0$, and therefore 
\begin{equation*}
    \widehat{\mathfrak C}_{\mathrm{KCI}} = \expect\left[
        \widehat b_{\A\given\C}
        \otimes
        \phi_{\c}(\C)
        \otimes
        \widehat b_{\B\given\C}
    \right]\,,
\end{equation*}
leading to a test statistic that might deviate from zero:
\begin{equation}
    \|\widehat{\mathfrak C}_{\mathrm{KCI}}\|_{\mathrm{HS}}^2
    = \expect\Big( \dotprod{\widehat b_{\A\given \C}(\C)}{\widehat b_{\A\given \C}(\C')}_{\mathcal{H}_{\a}}   k_{\c}(\C, \C')\dotprod{\widehat b_{\B\given \C}(\C)}{\widehat b_{\B\given \C}(\C')}_{\mathcal{H}_{\b}}\Big)\,,
    \label{eq:kci_hat_bias}
\end{equation}
where $\C,\C'$ are independent copies.

Thus, if the bias is ignored, the statistic will appear to indicate dependence, even though none exists---potentially leading to excessive Type I error rates.
For CIRCE, the $\langle \widehat b_{\A\given \C}(\C),\,\widehat b_{\A\given \C}(\C')\rangle_{\mathcal{H}_{\a}}$ term in the last equation is replaced
by $\dotprod{\muac(\C)}{\muac(\C')}_{\mathcal{H}_{\a}}$, yielding a potentially larger term when the regression bias is small. 
We can view the larger bias of CIRCE through the lens of double robustness \citep{chernozhukov2018double} of KCI: since conditional dependence is measured in terms of a product of the residuals, KCI only needs to get both regressions somewhat right, while CIRCE needs the second regression to be very accurate to reduce bias.

To mitigate the bias in KCI estimation, we first show that various KCI-like estimates of conditional independence asymptotically behave the same way:

\begin{restatable}{theorem}{splitkci}
    \label{th:splitkci}%
    For functions $\beta^{(1)},\beta^{(2)}: \mathcal{\C} \to \mathcal{H}_{\a}$ and $\tau^{(1)},\tau^{(2)}: \mathcal{\C} \to \mathcal{H}_{\b}$ bounded uniformly in RKHS norm across train set sizes, define the test statistic $T$ as a Hilbert-Schmidt inner product 
    
    \begin{equation*}
        T(\beta^{(1)},\tau^{(1)},\beta^{(2)},\tau^{(2)}) = \dotprod{\mathfrak C(\beta^{(1)}, \tau^{(1)})}{\mathfrak C(\beta^{(2)}, \tau^{(2)})}_{\mathrm{HS}}\,.
    \end{equation*}
    
    Under $\hzero$, $T \to 0$ if any one of $\beta^{(1)},\tau^{(1)},\beta^{(2)},\tau^{(2)}$ converges to the corresponding $\muac$ or $\mubc$ in RKHS norm.  

    Under $\hone$, $T \to t > 0$ if either both $\beta^{(1)},\beta^{(2)}\rightarrow \muac$ or both $\tau^{(1)},\tau^{(2)}\rightarrow \mubc$ in RKHS norm. 
\end{restatable}%
The proof, in \cref{app:proofs:splitkci}, shows equivalence between this measure and KCI/CIRCE.
The empirical estimator of $T$ over $n$ points follows the standard (quadratic-time) HSIC estimator (with $O(n^{-1/2})$ convergence; \citealp[see, e.g.,][Corollary C.6]{pogodin2022efficient}):
\begin{equation}
    T_n =
    \frac{1}{n^2}
    1_n\T
    \big(
    (HK^c_{\A}H) \odot K_{\C} \odot K_{\B}^c
    \big) 1_n
    \label{eq:kci_estimator}
,\end{equation}
where
$A \odot B$ is the elementwise product,
$H = I_n - \frac{1}{n} 1_n 1_n\T$ is the ``centering matrix'', and $(K_{\C})_{ij}=k_{\c}(\c_i, \c_j)$. 
In general, $K_{\A}^c$ (and, analogously, $K_{\B}^c$) is defined as 
    \begin{equation}
        (K^c_{\A})_{ij} = \dotprod{\phi_{\a}(\a_i)
    - \beta^{(1)}(\c_i)}{\phi_{\a}(\a_j)
    - \beta^{(2)}(\c_j)}\,.
    \label{eq:splitkci_matrix}
    \end{equation}
For KCI, all parts of the test statistic use CME estimates, i.e. $\beta^{(1)}=\beta^{(2)}=\hatmuac$ and $\tau^{(1)}=\tau^{(2)}=\hatmubc$. For CIRCE, $\beta^{(1)}=\beta^{(2)}=0$ and therefore $(K_\A)_{ij} = k_\a(\a_i, \a_j)$.

Furthermore, to work with symmetric matrices, in experiments we will symmetrize the measure as $\frac{1}{2}(T(\beta^{(1)},\tau^{(1)},\beta^{(2)},\tau^{(2)}) + T(\beta^{(2)},\tau^{(2)},\beta^{(1)},\tau^{(1)}))$.

As a consequence of \cref{th:splitkci}, we can choose how to estimate CME for each part of $T(\dots)$ to minimize the bias term under the null (\cref{eq:kci_hat_bias}). For that, we define SplitKCI:
\begin{definition}[SplitKCI]\label{def:splitkci}
    For CME estimates $\hatmuac^{(1)},\hatmuac^{(2)}$ trained over equal non-overlapping splits of a dataset, and analogously for $\hatmubc$, we define SplitKCI as 
    \begin{align}
        \mathrm{SplitKCI} = T\left(\hatmuac^{(1)},\ \hatmubc^{(1)},\ \hatmuac^{(2)},\ \hatmubc^{(2)}\right)\,\label{eq:splitkci_ab}
    \end{align}
\end{definition}

\Cref{eq:splitkci_ab} satisfies the conditions of \cref{th:splitkci}, meaning it asymptotically characterises conditional independence.

To analyse SplitKCI, we make standard assumptions to guarantee the CME estimator's convergence. For a space of real-valued Lebesgue square integral (with respect to measure $\pi$) functions $L_2(\pi)$ and an embedding $I_{\pi}\,:\,\mathcal{H}_\a\rightarrow L_2(\pi)$ that maps $f\in\mathcal{H}_\a$ to its $\pi$-equivalence class $[f]$, we can define interpolation spaces and the corresponding norm as follows:
\begin{definition}[$\beta$-interpolation space, \citealp{steinwart2012mercer,li2023optimal}]\label{def:beta_norm}
    For $\beta \geq 0$, a countable index set $I$, a non-increasing sequence $(\mu_i)_{i\in I}>0$, and a family $(e_i)_{i\in I}\in\mathcal{H}_\a$, such that $([e_i])_{i\in I}$ is an orthonormal basis (ONB) of $\overline{\mathrm{ran}\, I_{\pi}}\subseteq L_2(\pi)$ and $(\mu_i^{1/2}e_i)_{i\in I}$ is an ONB of $(\ker I_\pi)^{\perp}\subseteq \mathcal{H}_\a$,
    the $\beta$-interpolation space is defined as
    \begin{equation*}
        [\mathcal{H}]^\beta_\a = \left\{\sum_{i\in I}a_i\mu_i^{\beta/2}[e_i]\,:\,(a_i)_{i\in I} \in \ell_2(I)\right\}\in L_2(\pi)\,.
    \end{equation*}
    It is equipped with the $\beta$-power norm
    \begin{equation*}
        \norm{\sum_{i\in I}a_i\mu_i^{\beta/2}[e_i]}_\beta  = \norm{(a_i)_{i\in I}}_{\ell_2(I)} = \brackets{\sum_{i\in I}a_i^2}^{1/2}\,.
    \end{equation*}
\end{definition}

In particular for $\beta=0$ we have $\norm{\cdot}_0=\norm{\cdot}_{L_2(\pi)}$, and for $\beta=1$ we go back to the standard RKHS norm, i.e. $\norm{[f]}_1=\norm{f}_{\mathcal{H}_\a}$ for $f\in \mathcal{H}_\a$ (and $[f]=I_{\pi}(f)\in L_2(\pi)$). For $\beta>1$, we get a subspace of the original space with smoother operators.

For Hilbert-Schmidt spaces $\mathrm{HS}(\mathcal{H}_\a,\mathcal{H}_\b)$, we can similarly define vector-valued interpolation spaces $[\mathrm{HS}]^{\beta}$ (and the corresponding norm) by defining a mapping to $L_2(\pi,\mathcal{H}_\b)$. See Def. 3 of \citet{li2023optimal} for the full definition.

\begin{assumption}[CME rates, \citealp{fischer2020sobolev}]
    (EVD) The eigenvalues $\mu_i$ of $C_{\C\C}=\EE{\phi_{\c}(\C)\otimes\phi_{\c}(\C)}$ decay as $\mu_i\leq c_\mu \, i^{-1/p}$, for some $c_\mu>0$ and $p\in(0,1]$.
    (Boundedness) $k_{\c}(\C, \C') \le 1$.\footnote{This satisfies the EMB condition of \citet{fischer2020sobolev}, with $\alpha=1$ and $A=1$.}
    (SRC) The conditional mean operators satisfy $\muac\in[\mathrm{HS}_{\a\c}]^{\beta}, \mubc\in[\mathrm{HS}_{\b\c}]^{\beta}$ for the same $\beta\in[1, 2]$.\footnote{While $\beta>2$ is possible, convergence doesn't improve due to the saturation effect for ridge regression \citep{li2022saturation}.}
    \label{assumpt:cme}
\end{assumption}
We use the same $\beta$ for both operators only for simplicity; all proofs generalize to different interpolation space parameters $\beta_{\a\c},\beta_{\b\c}$.

\begin{restatable}[Bias in KCI and SplitKCI]{theorem}{kcibias}
Under \cref{assumpt:cme} and using independent data splits $\datatrain{m_{\a\c}}$ (for KCI), $\datatrain{\frac{m_{\a\c}}{2}(1)},\datatrain{\frac{m_{\a\c}}{2}(2)}$ (for SplitKCI), and analogously (and independently) for the $\C\rightarrow\B$ regression, the biases defined as the expectation over both the train and test data $b=\expect\, T(\dots)$ under $\hzero$ are bounded as
\begin{gather*}
    b_{\mathrm{KCI}}
    \leq
        ((1\!+\!K_1)\,\|\muac\|_{\beta}^2 \!+\! K_2)
        \, ((1\!+\!K_1)\,\|\mubc\|_{\beta}^2 \!+\! K_2)
        \, m_{\a\c}^{-\frac{\beta-1}{\beta+p}}
        \,  m_{\b\c}^{-\frac{\beta-1}{\beta+p}}
    \,,\\
    b_{\mathrm{SplitKCI}}
    \leq
        4\,
        \|\muac\|_{\beta}^2
        \,
        \|\mubc\|_{\beta}^2\, m_{\a\c}^{-\frac{\beta-1}{\beta+p}}
        \, m_{\b\c}^{-\frac{\beta-1}{\beta+p}}
    \,.
\end{gather*}
for $K_1,K_2,K_3 \!>\!0$ (independent of $m_{\a\c},m_{\b\c}$).
\label{th:kci-bias}
\end{restatable}
\begin{proof}[Proof idea]
The proof, in \cref{app:proofs:bias}, compares the bias terms using the CME estimator rates due to \citep{li2023optimal}. The KCI bound is much larger due to the constants in the upper bounds of the CME estimator -- realistically, these constants are unlikely to be tight.
The $b_{\mathrm{KCI}}$ bound is a sum of two terms (with a common multiplier):
$\|\muac\|_{\beta}^2$, coming from the bias due to non-zero ridge parameter $\lambda$,
and a (much larger due to large $K_1$) term
$K_1\,\|\muac\|_{\beta}^2 + K_2$,
which comes from $\dotprod{\hatmuac(\c_i)}{\hatmuac(\c_j)}$.
For SplitKCI, the latter term is zero, as we use two independent CME estimators in $\dotprod{\hatmuac^{(1)}(\c_i)}{\hatmuac^{(1)}(\c_j)}$; the former one is doubled since each regression uses $\frac{m_{\a\c}}{2}$ points and $2^{\frac{\beta-1}{\beta+p}} < 2$.
\end{proof}

This change is somewhat counter-intuitive: finding the CME is hard, so we want to use as much data as we can.
But due to the slow scaling with $m_{\a\c}$, halving the data doesn't hurt the overall estimate as much as decorrelating the errors improves it.
Even for SplitKCI, however,
the remaining bias term is still large unless $m_{\a\c}$ and $m_{\b\c}$ are large. This approach is similar to double cross-fitting of estimators for GCM-like measures \citep{mcclean2024double}, although in our case we have to use independent datasets for estimating the same regression twice (as opposed to fitting $\C\rightarrow\A$ and $\C\rightarrow\B$ regressions on independent datasets).

\subsection{A heuristic for choosing train/test split ratio for SplitKCI}\label{subsec:ttsplit}
So far the discussion on de-biasing of KCI has not included the crucial point we made in the introduction: for kernel-based methods, CME estimation convergence much slower than the test statistic itself: in the best case for CME, only $O(n^{-1/4})$ convergence is possible \citep{li2023optimal} (and worst-case convergence can be arbitrarily slow), as opposed to $O(n^{-1/2})$ convergence of HSIC-based statistics for a given CME \citep{gretton2007kernel}. 

Thus, if we use the same number of data points for both CME estimation and evaluation of the test statistic, we might obtain a test statistic that treats errors in the CME estimate as a significant correlation of residuals, leading to rejection of the null. \changed{In addition, using the training points (or a subset of those) for testing can introduce unexpected correlations into the test statistic (we discuss this in \cref{app:subsubseq:control}).} Therefore, we need to balance the number of train (for CME) and test (for the test itself) data points.

To the best of our knowledge, previous literature on KCI-based tests either did not use a train/test split at all, or used manually selected splits, which is impractical given that splitting influences both Type I and II errors.

One way to overcome the slow convergence rates of CME estimators is to use auxiliary data to estimate one regression, e.g. $\mubc$, very well. In some cases, such as genetic studies \citep{candes2018panning}, auxiliary $(\B,\C)$ data is available. 
Another option is to use the same data as both the train and test sets, as done in other tests, since the in-sample, i.e. train, regression errors are smaller than out-of-sample. This is the approach taken for GCM (\changed{see the discussion preceding Th. 8 of \citealt{shah2020hardness}} and further discussed by \citealt{lundborg2022projected}) and the original KCI paper \citep{zhang2012kernel}. However, since CME estimation convergence is much slower for kernel-based tests than for tests with finite-dimensional CMEs, like GCM, in-sample regression errors can still be too large to achieve level control; this is confirmed in experiments below. 

Therefore, we require a way to choose train and test data splits.
We outline two requirements for the train/test split procedure. For Type I error control, at least one CME estimator, $\hatmuac$ or $\hatmubc$, should be able to show independence from $\C$ for a chosen number of test points. For maximizing test power, we need to use as many test points as possible (while preserving the first quality).

To address the first point, we first note that for the true CME $\muac$, the $\phi_\a(\A) - \muac(\C)$ residuals should be independent of $\C$:
\begin{equation*}
    \mathfrak C_{\mathrm{train/test\ split}}^{\A\C} = \EE{(\phi_\a(\A) - \muac(\C))\otimes \phi_\c(\C)} = 0\,.
\end{equation*}
For a given train/test split, we can evaluate $\|\widehat{\mathfrak C}_{\mathrm{train/test\ split}}^{\A\C}\|^2$ (with or without train splitting into two regressions) and then compute the $p$-value for $\hzero: \|\mathfrak C_{\mathrm{train/test\ split}}^{\A\C}\|^2=0$.

We use this idea to build a heuristic for evaluating a given train/test split ratio. While this is a heuristic, we extensively evaluate it on several tasks in \cref{sec:experiments}. To choose the train/test split, we
\begin{enumerate}
    \item compute $\|\mathfrak C_{\mathrm{train/test\ split}}^{\A\C}\|^2$ and $\|\mathfrak C_{\mathrm{train/test\ split}}^{\B\C}\|^2$;
    \item compute $p$-values $p_{\A}$ and $p_{\B}$ from wild bootstrap (see \cref{subseq:wild});
    \item repeating this $r$ times, estimate the rejection rate $\omega = \frac{1}{r}\sum_i \II{p_{\A}^i \leq \alpha}\II{p_{\B}^i \leq \alpha}$.
\end{enumerate}
If $\omega$ is larger than the level $\alpha$, we conclude that the CME estimates are not accurate enough to detect conditional independence. Then, we change the split ratio to use fewer test points and repeat the procedure. The exact algorithm used in the experiments is summarized in \cref{alg:train_test_split} (in \cref{app:sec:kci}).

\section{Statistical testing with SplitKCI}\label{sec:testing}

To use KCI-like measures for statistical testing, we need to compute (approximate) $p$-values for the data. This is not straightforward: under the null, and assuming perfect estimates of $\muac$ and $\mubc$, these statistics are distributed as an infinite weighted sum of chi-squared variables \citep{zhang2012kernel,fernández2022general}, rather than the Gaussian that is often the case with other statistics. Moreover, the parameters of this mixture depend on the data distribution and the kernels, and so has to be approximated. The original KCI paper used a moment-matched Gamma approximation \citep{zhang2012kernel}, which lacks formal guarantees. In our preliminary experiments, this approximation tended to not produce correct $p$-values (\cref{fig:task1_gamma} in \cref{app:sec:kci}), so we turn to the wild bootstrap approximation \citep{wu1986jackknife,shao2010dependent,fromont2012kernels,leucht2013dependent,chwialkowski2014wild}.

\subsection{Wild bootstrap for computing p-values}\label{subseq:wild}

Following \citet{chwialkowski2014wild}, we provide wild bootstrap guarantees for an uncentered V-statistic version $\widehat V$ of \eqref{eq:kci_estimator}
(as well as analogous versions for CIRCE and SplitKCI), 
defined as
\begin{equation}
    \label{eq:psi_q_new}
    \widehat V=\frac{1}{n^2}
    1_n\T \sqbrackets{(qq\T) \odot \widehat K^c_{\A} \odot K_{\C}\odot \widehat K_{\B}^c} 1_n,
\end{equation}
where $q = 1_n$ is a vector of all ones.
We are justified in removing the centring from \cref{eq:kci_estimator}
because the estimator is asymptotically centred on the training data.
If $q$ is instead a vector of Rademacher variables ($\pm 1$ with equal probability), we refer to the wild bootstrap quantity \eqref{eq:psi_q_new} as $\widehat V^*$.

The following theorem guarantees the validity of the wild bootstrap for KCI, CIRCE, and SplitKCI, taking into account the estimation error of the CME.
To the best of our knowledge, this result is novel, including for the KCI.

\begin{restatable}[Wild bootstrap]{theorem}{wild}
    \label{theorem:wild_bootstrap}
    Assume that $m_{\a\c}$ and $m_{\b\c}$ %
    dominate $n^{\frac{\beta+p}{\beta-1}}$ asymptotically
    for $\beta$ and $p$ as in \cref{assumpt:cme}, and that $k_{\a}$, $k_{\b}$ and $k_{\c}$ are bounded by 1.
    (i) Under the null, $n\widehat V^*$ and $n\widehat V$ converge weakly to the same distribution.
    (ii) Under the alternative, $\widehat V^*$ converges to zero in probability, and $\widehat V$ converges to a positive constant.
    \label{th:wild}
\end{restatable}
The proof (in \cref{app:proofs:wild-bootstrap}) generalizes the techniques developed by \citet{chwialkowski2014wild} for HSIC to KCI, CIRCE, and SplitKCI, but takes into account imperfect CME estimates.

\Cref{theorem:wild_bootstrap} shows (pointwise)
asymptotic validity of the wild bootstrap, taking into account the estimation error of the CME.
Theorem \ref{theorem:wild_bootstrap}(i) shows that the test is asymptotically \emph{well-calibrated}: the Type I error is asymptotically controlled at the desired level.
That $n \hat V^*$ converges to the right distribution for a single data sample follows analogously to Theorem 3.1 of \citet{leucht2013dependent}; see also their Algorithm.
Theorem \ref{theorem:wild_bootstrap}(ii) shows that the test is \emph{consistent}: the Type II error of any alternative tends to zero asymptotically. 

\subsection{Full test for SplitKCI} 

The full testing procedure using SplitKCI remains similar to KCI: for a chosen set of kernels over $\A,\B,\C$, we first compute CME estimates, then compute CME-centred estimates of the kernel matrices, and finally construct the test statistic and find the approximate p-value of this static. Unlike for KCI, we use the train/test split heuristic of \cref{subsec:ttsplit} to compute CME estimates over the train set and kernel matrices over the test set; for a given training set, we compute two CME estimates using equal parts of the training set (the ``Split'' in SplitKCI) to reduce estimation bias in KCI; finally, we compute p-values using wild bootstrap (\cref{subseq:wild}). The full test procedure is summarized in \cref{alg:splitkci}. Note that the estimator \eqref{eq:kci_estimator} can be replaced with an unbiased estimate \eqref{app:eq:hsic_ub}, which we found gives a slight improvement. For SplitKCI$_{\A\ \mathrm{only}}$, $\hatmubc$ is computed using the full train set $\datatrain{m}$.

\begin{figure}[h]
\centering
\input{paper_content/algorithms/splitkci_alg}
\end{figure}

\section{Experiments}\label{sec:experiments}
We perform an extensive analysis of SplitKCI and competing tests in several tasks and testing scenarios. We use two synthetic and one real datasets of various complexity and dimensionality, described in \cref{subseq:task_description}. For the first low-dimensional synthetic dataset, we analyse the train/test splitting heuristic (introduced \cref{subseq:ttsplit}) in \cref{subseq:postnonlin_ttsplit}, and then evaluate all methods on the task in \cref{subseq:postnonlin_res}. We repeat the analysis of the heuristic on high-dimensional synthetic neural data in \cref{subseq:ttsplit}, and in \cref{subseq:levelpower,subseq:auxdata} compare the methods' performance with different amounts of available data. We conclude the experimental section with analysis of real insurance data in \cref{subseq:real_data}.
Code for all experiments is available at \href{https://github.com/romanpogodin/kernel-ci-testing}{\nolinkurl{github.com/romanpogodin/kernel-ci-testing}}. %

\subsection{Description of tasks and the experimental setup}\label{subseq:task_description}
\textbf{Synthetic data from a post-nonlinear model}
First, we evaluate the methods on data from a post-nonlinear model \citep{zhang2012identifiability}, in particular a variation of the model used in the original KCI paper \citep{zhang2012kernel}. This model generates $\A$ and $\B$ as nonlinear random (but fixed for each trial) noisy functions of $\C$. While $\A$ and $\B$ are one-dimensional, $\C$ can be high-dimensional, allowing to assess robustness of testing methods to dimensionality of $\C$. 

We generate the data for a $d$-dimensional $\C\sim \NN\brackets{0, 1/d}$ as follows:
\begin{align*}
    \A = G_\a\brackets{\frac{1}{d+1}\brackets{E_\a + \sum_{i=1}^dF_{\a,i}(\C_i + \xi_{\a,i})}},\\ \B = G_\b\brackets{\frac{1}{d+1}\brackets{E_\b + \sum_{i=1}^dF_{\b,i}(\C_i + \xi_{\b,i})}}\,,
\end{align*}
where $E_\a,E_\b\sim\NN(0, 1)$, all $G$ and $F$ are random functions (sampled as random convex combinations of $f_1(x)=x$, $f_2(x)=\tanh(x)$, and $f_3(x)=\tanh(x^3)$).

We use $\xi_{\a,i}$ and $\xi_{\b,i}$ to ``hide'' conditional dependence in the first coordinate. In all cases, $\xi_{\a,i}=\xi_{\b,i}=0$ for $i > 1$.
Under $\hzero$, $\xi_{\a,0},\xi_{\b,0}\sim\NN(0, 1)$ (independently). Under $\hone$, $\xi_{\a,0}=\xi_{\b,0}\sim\NN(0, 1)$. Thus, as dimensionality of the problem increases, it becomes progressively harder to detect conditional dependence.

\textbf{Synthetic neural data}  
We evaluate the tests on high-dimensional synthetic data that mimics neural recordings. We use the RatInABox toolbox \citep{George2024}, which models behavior and neural activity of a rat freely running in a generated 2d environment. RatInABox can model several types of neurons observed experimentally; we use the \href{https://github.com/RatInABox-Lab/RatInABox/blob/main/demos/conjunctive_gridcells_example.ipynb}{head direction and conjunctive cell models} (see \cref{app:experiments} for the exact model) as they generate non-trivial dependencies in the data. As $\A$, we use head direction cells: each has a preferred orientation of the rat's head, and is only active around those orientations \citep{knierim1995place}. As $\B$, we use conjunctive cells: each has a preferred head direction similar to $\A$, but is only active in certain locations in the environment \citep{sargolini2006conjunctive}, which can be obtain as a combination of position-only grid cells and head direction cells. For both cell types, their activity can be well explained by the animal's head direction and position, which we use as the conditioning variable $\C$.
This setup is similar to the \href{https://github.com/RatInABox-Lab/RatInABox/blob/main/demos/conjunctive_gridcells_example.ipynb}{RatInABox conjunctive cells demo}, and is visualised in \cref{fig:rat_vis}.

\begin{figure}[ht]
     \centering
     \includegraphics[width=0.95\textwidth]{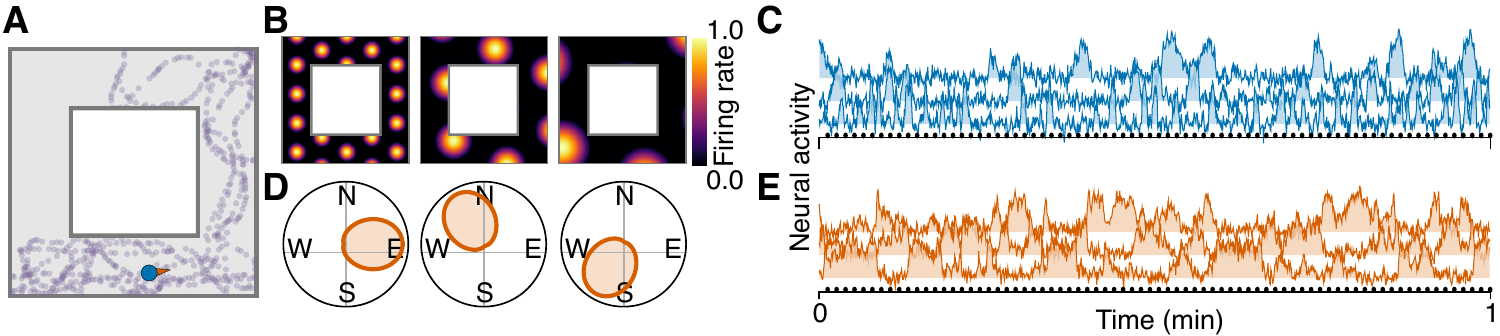}
     \caption{RatInABox data visualisation. \textbf{A.} Simulated rat trajectory in a square box with a blocked off centre. \textbf{B.} Activation patterns of three grid cells w.r.t. rat's position in the box. \textbf{C.} Simulated neural activity of the cells in \textbf{B}. Black dots indicate which data points are used in the dataset. \textbf{D.} Activation patterns of three head direction cells w.r.t. rat's head direction (polar coordinates). \textbf{E.} Same as \textbf{C}, but for the head direction cells.}
     \label{fig:rat_vis}
\end{figure}

We test the following question: is head direction selectivity of conjunctive cells $\B$ driven by the head direction cells $\A$? In other words, under the null the head direction and the conjunctive cells are independent given the actual head direction (and other behavioural variables like position that might contribute to the noise of both populations). 
Under the alternative, the head direction cells modulate activity of position-selective cells, giving rise to conjunctive cells. %

We use 100 cells for $\A$ and $\B$, and a 4-dimensional $\C$ (head direction vector and position). To generate the data, we simulate the animal's behaviour for several minutes, and then subsample the activity to obtain approximately i.i.d. data (choosing the sampling rate based on the noise autocorrelation). Other experimental details are postponed to \cref{app:experiments}.

\textbf{Car insurance data} 
Following the experiments of \cite{polo2023conditional}, we test our methods on the car insurance dataset originally collected from four US states and multiple insurance companies by \cite{angwin2017minority}.\footnote{%
Data available from \url{https://projects.propublica.org/graphics/carinsurance}.
} 

The dataset has three variables: car insurance price $\A$, minority neighborhood indicator $\B$ (defined as more than 66\% non-white in California and Texas, and more than 50\% in Missouri and Illinois), and driver's risk $\C$ (with additional equalization of driver-related variables; see \citealt{angwin2017minority}).

\textbf{Alternative tests} We compare kernel-based methods to GCM \citep{shah2020hardness} and RBPT2 \citep{polo2023conditional}. Both are regression-based methods, so we use the same set of kernels over $\C$ as for KCI variants, and do not require knowledge of $P(\B\given \C)$. 
GCM looks at correlations between $\EE{\A\given \C}$ and $\EE{\B\given \C}$ (estimated through kernel ridge regression, but with finite-dimensional outputs). RBPT2 compares how well $\A$ can be predicted by $g(\B,\C)\!=\!\EE{\A\given \B,\C}$ vs. $h\!=\!\EE{g(\B,\C)\given \C}$, also relying on kernel ridge regressions. We describe these methods in detail in \cref{app:sec:rbpt2}.
For RBPT2, we found the original method to be heavily biased against rejection of the null; we found an analytical correction and called the method RBPT2' (see \cref{app:sec:rbpt2}). Finally, we don't show CIRCE performance as it significantly underperforms vanilla KCI (see \cref{fig:task1_gamma} in \cref{app:experiments}).

\textbf{Data regimes} We evaluate all tests in two different data regimes. \textit{Standard}: $(\A,\B,\C)$ points available in triplets. \textit{Auxiliary data}: $(\A,\B,\C)$ points along with an independent set of $(\B,\C)$ points.
In both regimes, we need the data to first estimate $\hatmuac$ and $\hatmubc$, and then to evaluate the test statistic. 

\subsection{Influence of train/test splitting on KCI-style methods for a post-nonlinear model 
}\label{subseq:postnonlin_ttsplit}
First, we evaluate the train/test splitting heuristic introduced in \cref{subsec:ttsplit} on the post-nonlinear model (see \cref{subseq:task_description}). For $\alpha=0.05$ (for both test level and split threshold), we evaluate KCI and $\mathrm{SplitKCI}$ (\cref{alg:splitkci}) for test to train ratios between 0.1 and 0.8, and datast sizes $N=200$ and $N=400$.
 
For both choices of $N$, Type I error is similar between SplitKCI and KCI (\cref{fig:kci_task_budget}A, top). In contrast,  the split rejection rate (\cref{fig:kci_task_budget}A, bottom) for SplitKCI reflects its Type I error much better than for KCI, as indicated by slower raise of the split rejection rate for higher test/train ratios. 

Type II error is similar between SplitKCI and KCI across all split ratios. However, the split rejection rate again better reflects the changes in Type II error for SplitKCI. 
Therefore, SplitKCI with the train/test splitting heuristic can choose split ratios that favour smaller Type II errors (also indicated by asterisks in \cref{fig:kci_task_budget}B placed roughly where the split reject rate crosses the $\alpha=0.05$ threshold).

\begin{figure}[ht]
     \centering
     \includegraphics[width=0.95\textwidth]{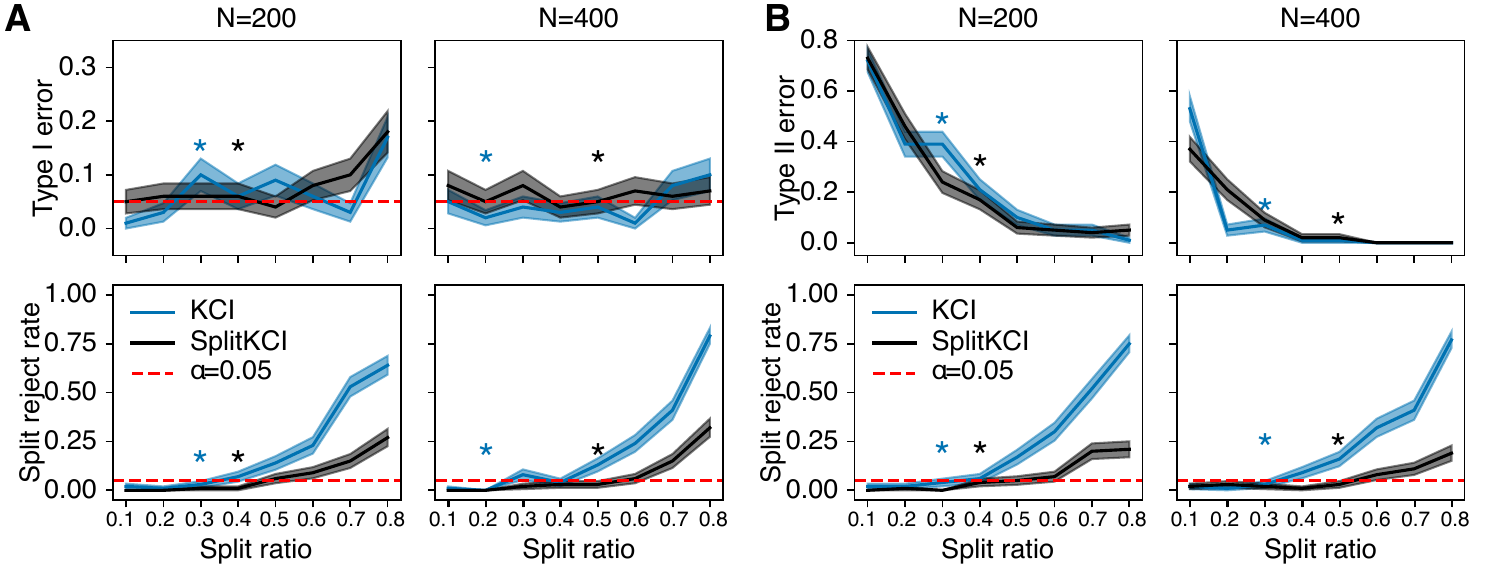}
     \caption{Train/test splitting in a post-nonlinear model with $d=4$ and a fixed dataset size $N=n+m$ (for $n$ test and $m$ training points). \textbf{A.} Type I error (top) vs. rejection rate for the train/test split heuristic (bottom) for different dataset sizes $N$ and test to train split ratios. \textbf{B.} Type II error (top) vs. rejection rate for the train/test split heuristic (bottom) for different dataset sizes $N$ and split ratios. Lines/shaded area: mean/$\pm$SE over 100 trials, $\alpha=0.05$. Asterisks: approximate split ratios from \cref{alg:train_test_split} to show performance of the splitting heuristic.}
     \label{fig:kci_task_budget}
\end{figure}

\subsection{Methods comparison for a post-nonlinear model}\label{subseq:postnonlin_res}

Next, we evaluate performance over varying number of data points $N=n+m$. For SplitKCI, we find the train/test split ratio using \cref{alg:train_test_split} for one random seed, and then evaluate the rest of the random seeds on the chosen split ratio. For KCI and RBPT2', we show results for $n=100$ test points. For GCM, we don't use splitting, as was recommended by the original paper \citep{shah2020hardness}.

Under $\hzero$ and across all dimensionalities and total number for data points, all methods generally hold level (\cref{fig:kci_task_best}A), although RBPT' and KCI are above level for some dimensionalities. As we determine the train/test split for SplitKCI using the proposed heuristic, this result shows that the heuristic can correctly find the ratio sufficient for the required Type I error.

Under $\hone$, for $N=200$ (\cref{fig:kci_task_best}B, left) KCI and SplitKCI show worse Type II error than the other two methods, although notably SplitKCI performs comparably to KCI. For $N=400$, SplitKCI performs significantly better than KCI and on part with the other two methods (\cref{fig:kci_task_best}B, right). First, this means that $N=400$ provides enough data for KCI-style methods to detect conditional dependence in this task. Second, this results shows that the train/test splitting heuristic is powerful enough to make use of the additional data, while the standard KCI with a fixed split is not.

In this task, not using a train/test split at all was a valid strategy not only for GCM, but also for KCI/SplitKCI (see \cref{app:fig:kci_task_no_split} in the Appendix): just like GCM, KCI/SplitKCI hold level and achieve nearly zero Type II error. However, as we will see in the next section, this strategy only works for simple tasks. 

\begin{figure}[ht]
     \centering
     \includegraphics[width=0.95\textwidth]{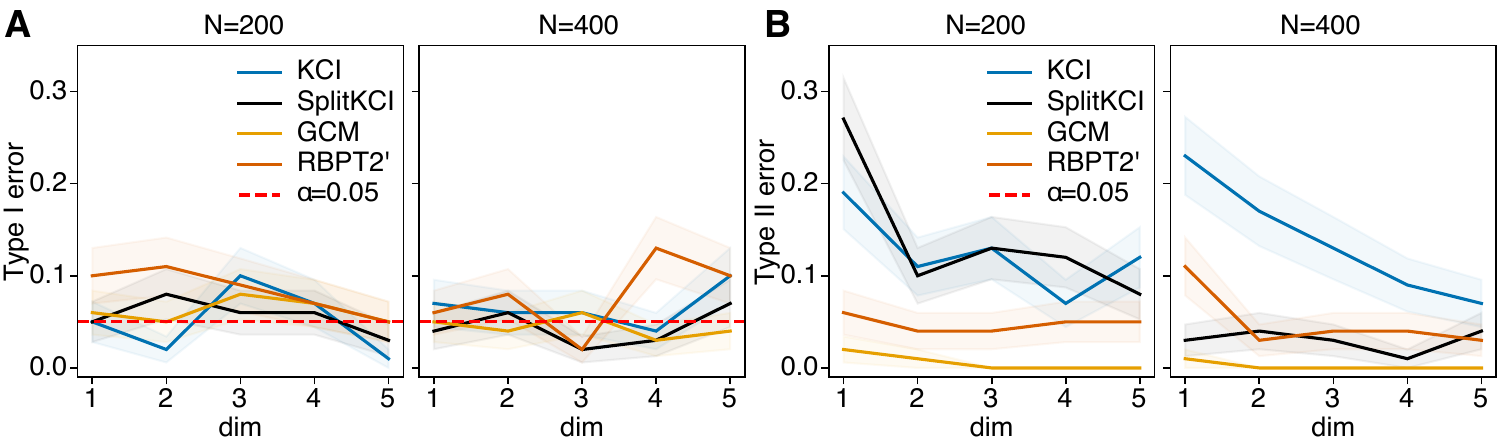}
     \caption{Post-nonlinear model experiments for increasing dimensionality of the task and a fixed dataset size $N=n+m$ (for $n$ test and $m$ training points). \textbf{A.} Type I error for $N=200$ (left) and $N=400$ (right) data points. \textbf{B.} Type II error for $N=200$ (left) and $N=400$ (right) data points. Lines/shaded area: mean/$\pm$SE over 100 trials, $\alpha=0.05$.}
     \label{fig:kci_task_best}
\end{figure}

\subsection{Influence of train/test splitting on KCI-style methods for synthetic neural data}\label{subseq:ttsplit}

For the synthetic neural data, we again first  we evaluate the train/test splitting heuristic introduced in \cref{subsec:ttsplit}. For $\alpha=0.05$ (for both test level and split threshold), we evaluate KCI and $\mathrm{SplitKCI}$ for test to train ratios between 0.1 and 0.8, for three dataset sizes (200, 500, 1000).

For all budget sizes, $\mathrm{SplitKCI}$ shows Type I error dynamics which is significantly better than for regular KCI (\cref{fig:budget_h0}, top). For KCI, the split rejection rate (\cref{fig:budget_h0}, bottom, blue lines) steadily increases with split size, meaning that it cannot serve as a reliable measure of Type I error. In contrast, for $\mathrm{SplitKCI}$ with 500 and 1000 data points (\cref{fig:budget_h0}, bottom, black lines), the split rejection rate can reliably select split ratios that result in Type I error control. 

For Type II error control, both methods show similar trends (\cref{fig:budget_h1}): $\mathrm{SplitKCI}$ can reliably select split ratios that result in small Type II error, while KCI cannot (which is also indicated by asterisks in \cref{fig:budget_h1} placed roughly where the split reject rate crosses the $\alpha=0.05$ threshold).

These results show interplay of several factors that affect SplitKCI performance. Compared to KCI, CME estimation bias has less effect on SplitKCI, leading to lower Type I errors. For the same reason, the proposed heuristic for choosing the train/test split ratio better reflects Type I performance. Finally, lower bias allows to use more test points, which in turn leads to lower Type II errors as the test becomes more sensitive to true conditional dependence.  More generally, these results confirm our discussion in \cref{sec:theory} about the effect of the number of train points $m$ and test points $n$: larger $n / m$ always leads to higher Type I errors.

\begin{figure}[ht]
     \centering
     \includegraphics[width=0.95\textwidth]{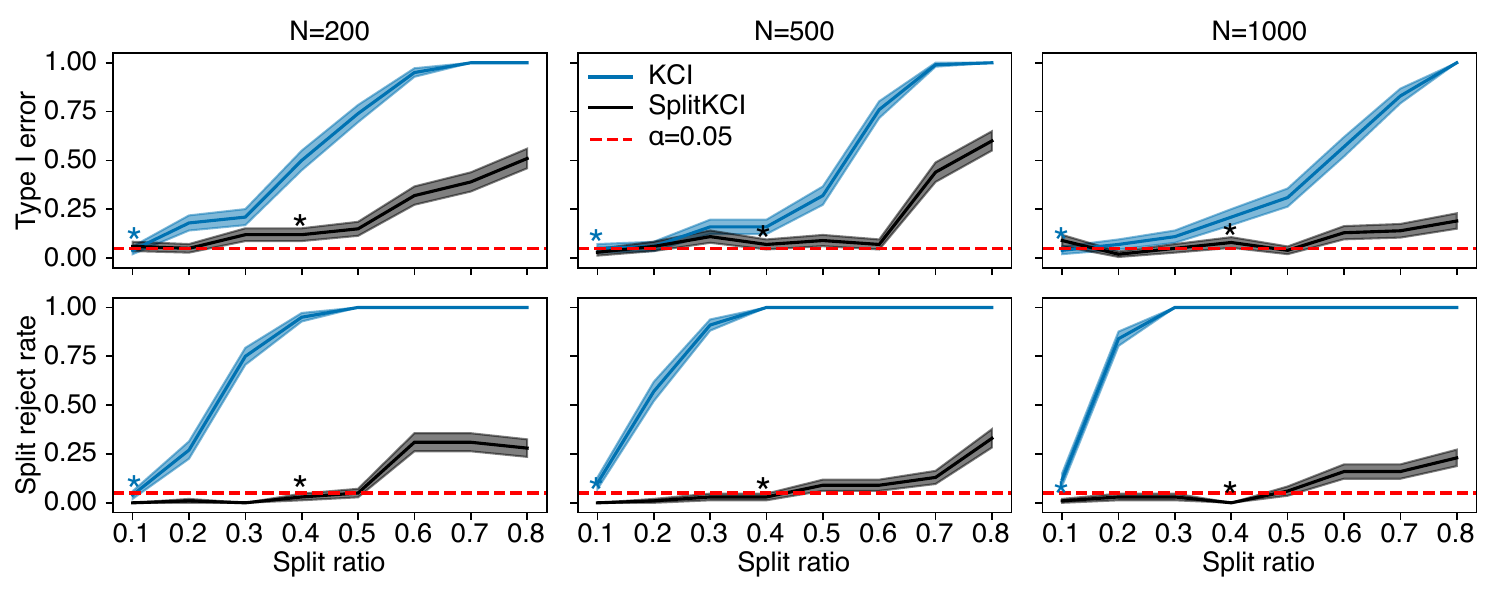}
     \caption{Train/test splitting in the synthetic neural data. Type I error (top) vs. rejection rate for the train/test split heuristic (bottom) for different dataset sizes $N=n+m$ and test to train ($n/m$) split ratios. Lines/shaded area: mean/$\pm$SE over 100 trials, $\alpha=0.05$. Asterisks: approximate split ratios from \cref{alg:train_test_split} to show performance of the splitting heuristic.}
     \label{fig:budget_h0}
\end{figure}

\begin{figure}[ht]
     \centering
     \includegraphics[width=0.95\textwidth]{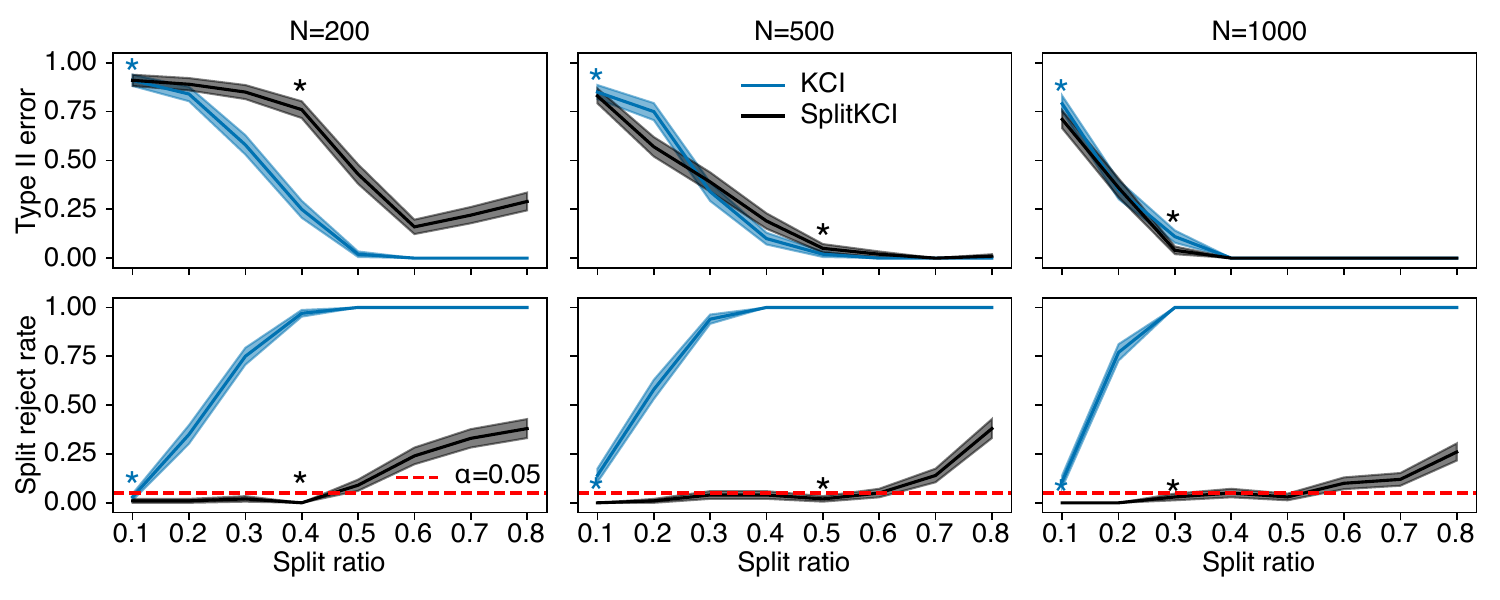}
     \caption{Train/test splitting in the synthetic neural data. Type II error (top) vs. rejection rate for the train/test split heuristic (bottom) for different dataset sizes $N=n+m$ and test to train ($n/m$) split ratios. Lines/shaded area: mean/$\pm$SE over 100 trials, $\alpha=0.05$. Asterisks: approximate split ratios from \cref{alg:train_test_split} to show performance of the splitting heuristic.}
     \label{fig:budget_h1}
\end{figure}

\subsection{SplitKCI holds level and achieves high power for synthetic neural data}\label{subseq:levelpower}

Here, we evaluate performance over varying number of data points $N=n+m$. For SplitKCI, we find the train/test split ratio using \cref{alg:train_test_split} for one random seed, and then evaluate the rest of the random seeds on the chosen split ratio. For KCI and RBPT2', we show results for $n=100$ test points. For GCM, we don't use splitting, as was recommended by the original paper \citep{shah2020hardness}. Other fixed splitting strategies for all methods perform as well or worse (see \cref{app:fig:rat_std_100_05,app:fig:rat_std_no_split_control}).

\paragraph{Type I error} Under the null, SplitKCI achieved level control, as well as GCM and (for larger $N$) RBPT2' (\cref{fig:rat_standard}A). KCI struggled much more, especially at smaller dataset sizes. This illustrates that the automated splitting procedure for SplitKCI allows the test to hold level without pre-determined train/test splits (as well as for fixed splits, see \cref{app:fig:rat_std_100_05}). 

\paragraph{Type II error} Under the alternative, all methods apart from KCI quickly achieved low error (\cref{fig:rat_standard}B). As we used a fixed size of the test dataset for KCI, increasing the total budget $N$ led to increasingly better CME estimates. Therefore, for larger $N$, KCI estimate becomes less biased, which in turns leads to less frequent rejection. If we increase the number of test points, the power would increase as well (as shown in \cref{fig:budget_h1}). However, KCI in this case would fail to hold level. In contrast, our train/test approach for SplitKCI allows to find splits that achieve high power while still holding level, even though the optimal split size varies across $N$ (see \cref{app:fig:rat_std_100_05} for $n=100$ and $n=m$ performance).  

Using the train/test splitting for KCI and SplitKCI was crucial for Type I performance, as hinted earlier in \cref{seq:regression_ci}. While GCM and many other tests often don't use a train/test split (with GCM achieving the best results in that setting in all of our experiments), KCI and SplitKCI without a train/test split do not achieve level $\alpha$ in this task (see \cref{app:subsubseq:control}).

\begin{figure}[ht]
     \centering
     \includegraphics[width=0.95\textwidth]{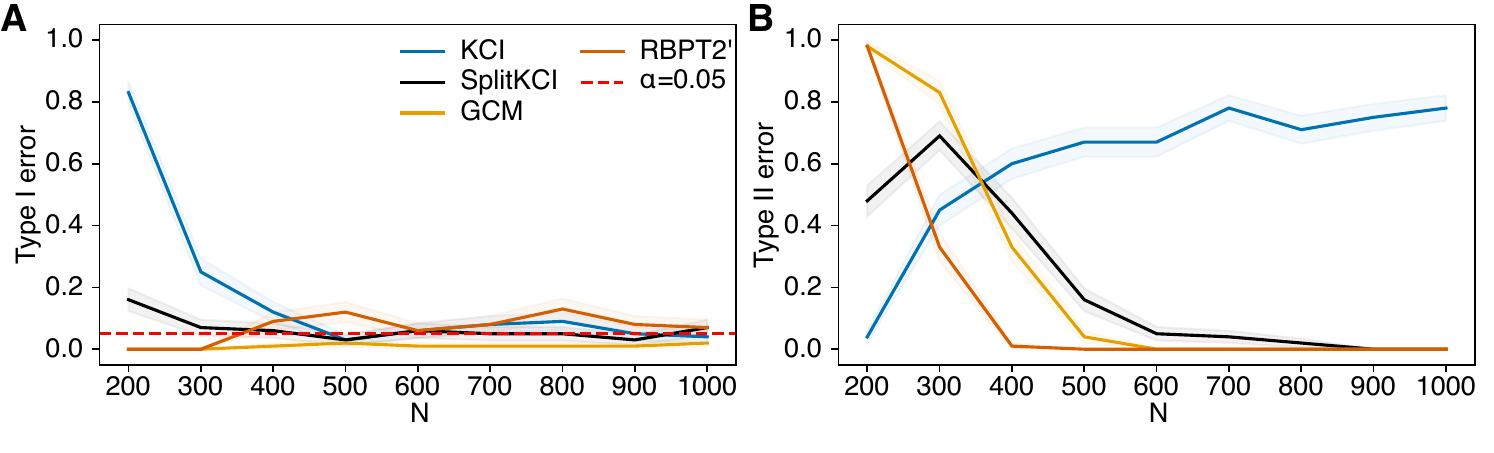}
     \caption{Synthetic neural data experiments for an increasing dataset size $N=n+m$ (for $n$ test and $m$ training points). \textbf{A.} Type I error. \textbf{B.} Type II error. Lines/shaded area: mean/$\pm$SE over 100 trials, $\alpha=0.05$.}
     \label{fig:rat_standard}
\end{figure}

\subsection{SplitKCI holds level and achieve high power with auxiliary data for synthetic neural data}\label{subseq:auxdata}
In some application, such as genetic studies \citep{candes2018panning}, full triplets of $(\A,\B,\C)$ points could be scarcely available, while $(\B,\C)$ auxiliary is not (e.g.\ because it is easier to collect). For the synthetic neural data, one could imagine that simultaneous recordings from two neural populations $\A$ and $\B$ are expensive, while collecting just one population $\B$ along with behavioural data $\C$ is cheaper.
As CME estimation bias for both KCI and SplitKCI depend on the size of the $(\B,\C)$ multiplicatively (see \cref{th:kci-bias}), we expect all tests to perform better in this regime.

We studied the setting in which the number of $(\A,\B,\C)$ points is fixed and small (we use $N=200$ and $N=400$ points), but we have access to an independent large set of $M=(\B,\C)$ points. We therefore use $N$ for the $\C\rightarrow\A$ regressions and testing, and $M$ points for the $\C\rightarrow\B$ regressions. We don't use a train/test split over $N$ for KCI, SplitKCI, and GCM, and use $n=100$ for RBPT'. 

\paragraph{Type I error} Under the null, all methods apart from RBPT2' held level without train/test splitting over $(\A,\B,\C)$ (\cref{fig:rat_aux}A,C; for RBPT2', a better Type I error was achieved for 100/100 train/test split). For SplitKCI, splitting did not affect Type I error much, but KCI was at level only without train/test splitting (see \cref{app:fig:rat_aux_200,app:fig:rat_aux_400}). A potential explanation is that due to small $m$, the $\C\rightarrow\A$ tends to overfit, producing small in-sample and large out-of-sample errors. (The effect on SplitKCI is smaller since the train data is split, and so half of the residuals for each regression are always out-of-sample.) 

\paragraph{Type II error} Under the alternative, for $N=200$ all methods generally showed poor Type II error (\cref{fig:rat_aux}B), with GCM having no power. For RBPT2', its better performance is consistent with poor Type I error control. For $N=400$ (\cref{fig:rat_aux}D), Type II error decreases for all methods, staying the lowest for RBPT' and SplitKCI, and the highest for GCM and KCI.

\begin{figure}[ht]
     \centering
     \smallskip
     \includegraphics[width=0.95\textwidth]{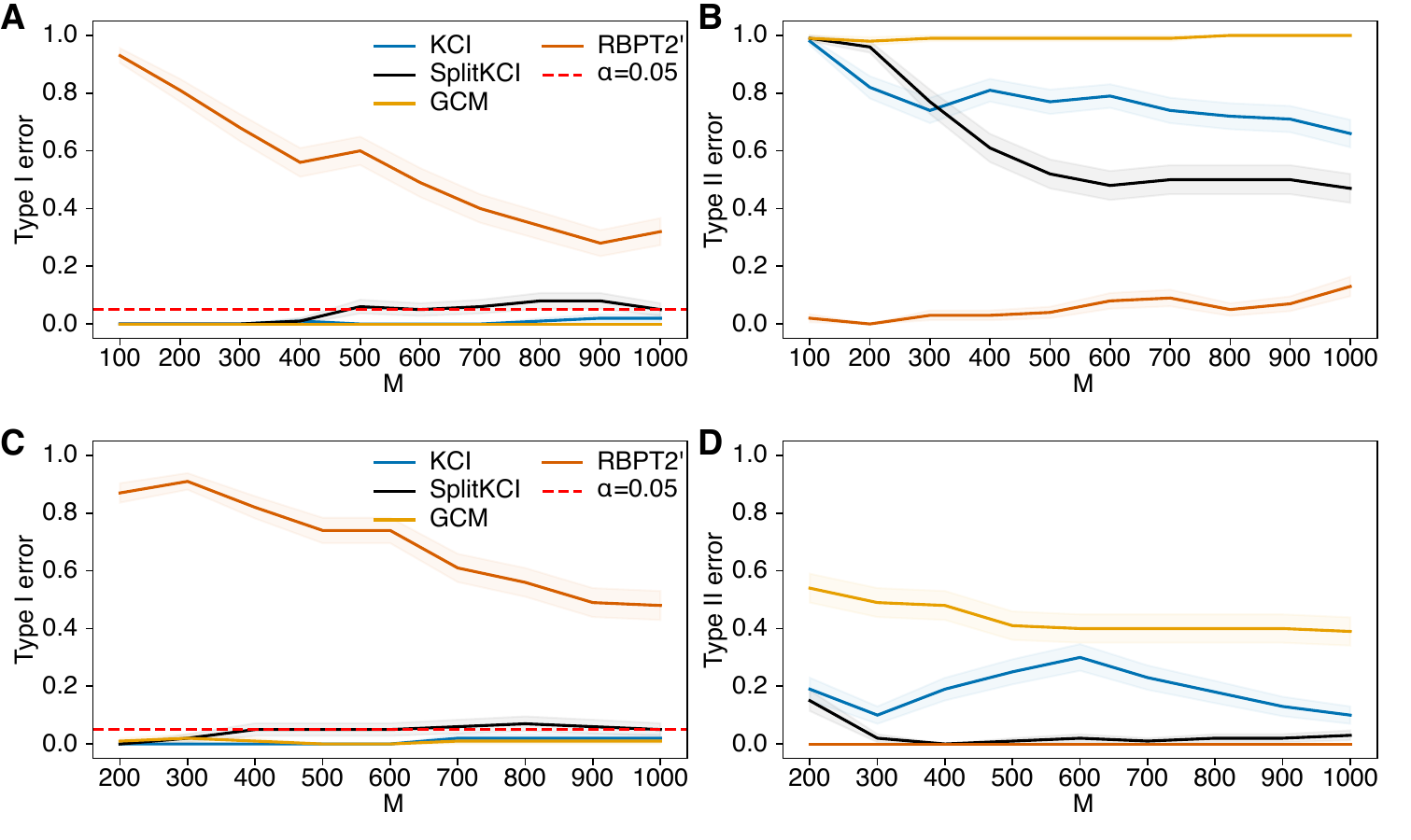}
     \caption{Synthetic neural data experiments for a fixed $(A,B,C)$ dataset size ($N$ points), but increasing amount of auxiliary $(B,C)$ data ($M$ points). \textbf{A.} Type I error for 200 $(A,B,C)$ points. \textbf{B.} Type II error. \textbf{C-D.} Same as \textbf{A-B}, but for 400 $(A,B,C)$ points. Lines/shaded area: mean/$\pm$SE over 100 trials, $\alpha=0.05$.}
     \label{fig:rat_aux}
\end{figure}

\begin{figure}[ht]
    \centering
    \includegraphics[width=0.95\textwidth]{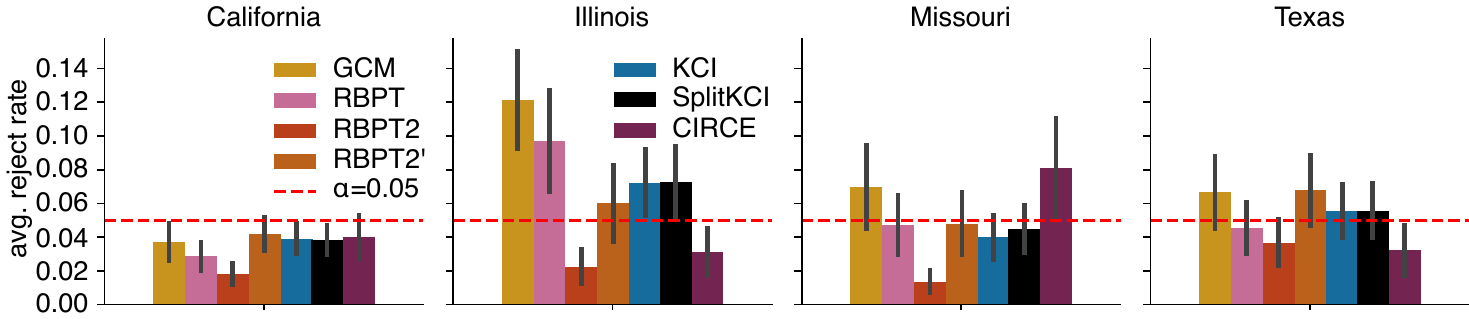}
    \caption{Simulated null hypothesis for the car insurance data. For each state, the average rejection rate across several companies is calculated as in \cite{polo2023conditional}. Mean $\pm$SE over 50 trials.}
    \label{fig:real_data}
\end{figure}

\phantom{dunno why the spacing gets all screwed up but this helps a bit}

\subsection{Car insurance data}\label{subseq:real_data}

We follow the data assessment experiment from \cite{polo2023conditional} to evaluate different methods under a \textit{simulated} $\hzero$. 
\citet{polo2023conditional} divided the driver's risk into 20 clusters, and shuffled the corresponding insurance price (per company and per state; about 1-2k points for each, which is enough to avoid merging train and test data). The resulting data was evaluated for several tests, with average (over companies) rejection rate reported as Type I error. We find (see \cref{fig:real_data}) that both KCI and SplitKCI produce robust results at the desired $\alpha\!=\!0.05$ level. As well, RBPT2 with or without bias correction holds level. In contrast, GCM, RBPT and CIRCE in some cases produce larger Type I errors.

Finally, evaluated all methods on the full dataset to test for conditional independence. Using a 70/30\% train/test split as in \cite{polo2023conditional} (with 24-34k points total per state), we computed $p$-values for each method. As in \cite{polo2023conditional}, RBPT and RBPT2 did not reject the null at $\alpha\!=\!0.05$ for California data, but did for the other states. All other methods, including the bias-corrected RBPT2', rejected the null for all states.
The results are summarized in \cref{fig:real_data_pval}, suggesting that most methods, apart from RBPT and RBPT2, reject the null at $\alpha=0.05$ for all states. CIRCE (not shown in the main text) performed similar to KCI/SplitKCI but produced less accurate Type 1 errors.

\begin{figure}[ht!]
    \centering
    \includegraphics[width=0.95\linewidth]{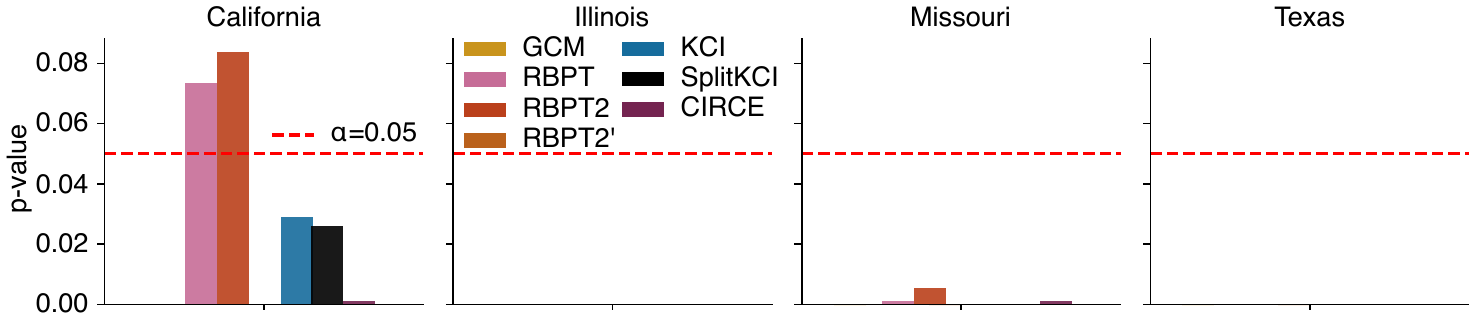}
    \caption{Evaluation of all methods on the full car insurance dataset (without averaging over companies, unlike in the simulated null in \cref{fig:real_data}). Expanded from \cite{polo2023conditional} by adding KCI-based methods.}
    \label{fig:real_data_pval}
\end{figure}

\section{Discussion}
\label{sec:discussion}

We introduced SplitKCI, a version of the KCI \citep{zhang2012kernel} test that is less biased towards rejection and can be combined with a train/test data splitting heuristic to achieve Type I error control under the null and also Type II error minimisation under the alternative. We also showed that the wild bootstrap approach for producing $p$-values for kernel-based independence tests can be generalised to KCI and SplitKCI with imperfect CME estimates. In experiments, we showed that SplitKCI can indeed hold level and produce lower Type II errors that KCI, and also than non-kernel based methods. 

Several extensions of our work are possible. Other kernel-based measures of conditional dependence \citep{fukumizu2007kernel,strobl2019approximate,park2020measure,scetbon2022asymptotic} also rely on CME and are limited by CME estimator's quality; our insights might be used to improve alternative methods. 

As CME estimation accuracy is crucial for test performance, parametric alternatives to kernel ridge regression might provide better estimates in some cases. For instance, when $\A$ is categorical, we can replace kernel ridge regression with logistic regression, or any classifier, with a linear kernel on top; this can produce accurate predictions of $\A$ that are not limited by simple hand-crafted kernels. It might also be possible to estimate $\muac$ with parametric methods e.g.\ as $\hatmuac(\c)=f_\theta(\c)\Phi_{\a}$ -- cf.\ \eqref{eq:cme} -- similar to regression-based methods like GCM \citep{shah2020hardness}.  

Next, our wild bootstrap result shows pointwise asymptotic level control. While it might be possible to show uniformly asymptotic level control for distributions satisfying our assumptions, we believe this is more challenging for kernel-based measures than for settings with asymptotically normal statistics (e.g.\ GCM). One approach could be to find an asymptotically normal (Split)KCI estimator, as in the cross-HSIC statistic for unconditional independence testing \citep{shekhar2023permutation}, but it is unclear how this drastically different approach would perform in practice.

Using a kernel-based wild bootstrap can lead to non-asymptotic Type I error guarantees in certain testing scenarios, such as MMD two-sample tests \citep{schrab2021mmd,schrab2022efficient,fromont2012kernels} and HSIC independence tests \citep{albert2022adaptive}.
In other cases, the control is only asymptotic, as in KSD goodness-of-fit \citep{schrab2022ksd,key2021composite}.
It might be possible to develop non-asymptotic results for KCI-style tests as well.
\citet{shah2020hardness} proved that if a conditional independence test always controls the Type I error non-asymptotically, then it cannot have power against \emph{any} alternative.
Throughout this paper, however, we assume that we can reliably estimate the conditional means -- that is, the underlying operators are smooth enough. This restricts the space of nulls and alternatives for which kernel-based methods work, suggesting that it might be possible to develop results similar to non-asymptotically valid tests that can achieve minimax-optimal power \citep{schrab2021mmd,biggs2023mmdfuse,kim2023differentially}.

Conditional dependence tests are often used for sensitive domains in which the conclusions of a study can have significant societal impacts
\citep[e.g.][]{angwin2017minority}.
A further important topic for future work is therefore test interpretability,
since conditional (in)dependence may be caused by unknown confounding variables.

\acks{This work was supported by
the Natural Sciences and Engineering Research Council of Canada
(NSERC Discovery Grant: RGPIN-2020-05105; Discovery Accelerator Supplement: RGPAS-2020-00031; Arthur B. McDonald Fellowship: 566355-2022),
the Canada CIFAR AI Chairs program,
a CIFAR Learning in Machine and Brains Fellowship,
U.K. Research and Innovation (grant EP/S021566/1),
and the Gatsby Charitable Foundation.
This research was enabled in part by support provided by Calcul Québec %
and the Digital Research Alliance of Canada. %
The authors acknowledge the material support of NVIDIA in the form of computational resources.

The authors would like to thank Namrata Deka and Liyuan Xu for helpful discussions.}

\newpage

\toggletrue{inappendix}
\appendix

We present the proofs of the technical results in \cref{app:proofs}.
Details about other measures of conditional independence are presented in \cref{app:sec:other_methods}.
\cref{app:experiments} contains experimental details, as well as additional experiments.

\section{Proofs}
\label{app:proofs}

\subsection{SplitKCI definition} \label{app:proofs:splitkci}
\splitkci*
\begin{proof}
    Under $\hzero$, assuming w.l.o.g. that $\beta^{(1)}\rightarrow\muac$ in RKHS norm and using the fact that other CME estimates are bounded,
        \begin{align*}
            T(\beta^{(1)},\tau^{(1)},\beta^{(2)},\tau^{(2)}) \rightarrow T(\muac,\tau^{(1)},\beta^{(2)},\tau^{(2)}) \\
            = \dotprod{\mathfrak C(\muac,\tau^{(1)})}{\mathfrak C(\beta^{(2)},\tau^{(2)})}\,.
        \end{align*}
    Using linearity of expectation, we have that
    \begin{align*}
        \mathfrak C(\muac,\tau^{(1)}) \,&= \expect\left((\phi_{\a}(\A)-\muac(\C))\otimes \phi_{\c}(\C)\otimes (\phi_{\b}(\B) - \tau^{(1)}(\C))\right)\\
        &=\expect\left((\phi_{\a}(\A)-\muac(\C))\otimes \phi_{\c}(\C)\otimes \phi_{\b}(\B)\right)=\mathfrak C(\muac, 0)=_{\hzero} 0\,,
    \end{align*}
    and therefore $T =_{\hzero} 0$.

    Under $\hone$, again assuming w.l.o.g. that $\tau^{(1)},\tau^{(2)}\rightarrow\mubc$ in RKHS norm and using the fact that other CME estimates are bounded,
   \begin{align*}
            T(\beta^{(1)},\tau^{(1)},\beta^{(2)},\tau^{(2)}) \,&\rightarrow T(\beta^{(1)},\mubc,\beta^{(2)},\mubc) \\
            &= \dotprod{\mathfrak C(\beta^{(1)}, \mubc)}{\mathfrak C(\beta^{(2)},\mubc)}\\
            &=\dotprod{\mathfrak C(0, \mubc)}{\mathfrak C(0, \mubc)} = \|\mathfrak C_{\mathrm{CIRCE}}\|_{\mathrm{HS}}^2 >_{\hone} 0\,.
    \end{align*}
\end{proof}

\subsection{Bias} \label{app:proofs:bias}
\kcibias*
\begin{proof}
First, we will consider the most generic case of SplitKCI, later assuming that $\hatmuac^{(1)}=\hatmuac^{(2)}$ and $\hatmubc^{(1)}=\hatmubc^{(2)}$ for KCI. Denoting 
\begin{align*}
    (\widehat K^c_{\A})_{ij} \,&= \frac{1}{2}(\widehat K^{12}_{\A})_{ij} + \frac{1}{2}(\widehat K^{21}_{\A})_{ij}\,,\quad (\widehat K^{12}_{\A})_{ij}=\dotprod{\phi_{\a}(\a_i)-\hatmuac^{(1)}(\c_i)}{\phi_{\a}(\a_j)-\hatmuac^{(2)}(\c_j)}\,,
\end{align*}
and similarly $(\widehat K^c_{\B})_{ij}$ computed over for \textit{test} points $i,j$. Additionally decomposing all CME estimates as
\begin{align*}
    (\widehat K_{\B\C\given \C})_{ij} \,&= (K_{\C}\odot \widehat K^c_{\B})_{ij} = k_{\c}(\c_i, \c_j)\,\dotprod{\phi_{\b}(\b_i)-\hatmubc(\c_i)}{\phi_{\b}(\b_j)-\hatmubc(\c_j)}\,,
\end{align*}
where we decompose the empirical estimate of the mean into the true one, the distance between the true one and the expected estimate given the ridge parameters, and finally between the ridge mean and its estimate (so that the last term is zero-mean).

Under the null, we can immediately remove the terms with the true centered variables $\phi_{\a}(\a_i)-\muac^{(1)}(\c_i)$:
\begin{align*}
    &\expect\, (\widehat K^{12}_{\A})_{ij}\,(K_{\C})_{ij}\,(\widehat K_{\B}^{12})_{ij} \\
    \,&=\expect\, \dotprod{\phi_{\a}(\a_i)-\hatmuac^{(1)}(\c_i)}{\phi_{\a}(\a_j)-\hatmuac^{(2)}(\c_j)} (K_{\C})_{ij}\,(\widehat K_{\B}^{12})_{ij}\\
    &=_{\hzero}\expect\, \dotprod{\delta\muac^{\lambda}(\c_i) + \delta\hatmuac^{(1)\lambda}(\c_i)}{\delta\muac^{\lambda}(\c_j) + \delta\hatmuac^{(2)\lambda}(\c_j)} (K_{\C})_{ij}\,(\widehat K_{\B}^{12})_{ij}\\
    &=\expect\, \sqbrackets{\dotprod{\delta\muac^{\lambda}(\c_i)}{\delta\muac^{\lambda}(\c_j)} + \dotprod{\delta\hatmuac^{(1)\lambda}(\c_i)}{\delta\hatmuac^{(2)\lambda}(\c_j)}} (K_{\C})_{ij}\,(\widehat K_{\B}^{12})_{ij}\,.\\
    &=_{\hzero} \expect\, \sqbrackets{\dotprod{\delta\muac^{\lambda}(\c_i)}{\delta\muac^{\lambda}(\c_j)} + \dotprod{\delta\hatmuac^{(1)\lambda}(\c_i)}{\delta\hatmuac^{(2)\lambda}(\c_j)}} k_{\c}(\c_i, \c_j)\dotprod{\delta\hatmubc(\c_i)}{\delta\hatmubc(\c_j)}\,,
\end{align*}
where in the third line we took the expectation over $\datatrain{\frac{m_{\a\c}}{2}(1)},\datatrain{\frac{m_{\a\c}}{2}(2)}$ using the facts that $\expect\,\delta\hatmuac^{1/2\lambda}=0$ and that all data splits (apart from $\datatrain{\frac{m_{\a\c}}{2}(1)},\datatrain{\frac{m_{\a\c}}{2}(2)}$) are independent. Repeating the same calculation for $(\widehat K_{\B}^{12})_{ij}$,
\begin{align*}
    &\expect\, (\widehat K^{12}_{\A})_{ij}\,(K_{\C})_{ij}\,(\widehat K_{\B}^{12})_{ij} \\
    \,&=\expect\, \left(\sqbrackets{\dotprod{\delta\muac^{\lambda}(\c_i)}{\delta\muac^{\lambda}(\c_j)} + \dotprod{\delta\hatmuac^{(1)\lambda}(\c_i)}{\delta\hatmuac^{(2)\lambda}(\c_j)}} k_{\c}(\c_i, \c_j)\right.\\
    &\qquad\qquad\qquad\left.\sqbrackets{\dotprod{\delta\mubc^{\lambda}(\c_i)}{\delta\mubc^{\lambda}(\c_j)} + \dotprod{\delta\hatmubc^{(1)\lambda}(\c_i)}{\delta\hatmubc^{(2)\lambda}(\c_j)}} \right)\,.
\end{align*}

For SplitKCI we can take the expectation w.r.t. the train data, removing the zero-mean $\dotprod{\delta\hatmuac^{(1)\lambda}(\c_i)}{\delta\hatmuac^{(2)\lambda}(\c_j)}$ term. Therefore, if we take expectations over the test points $i,j$ under $\hzero$,
\begin{align*}
    b_{\mathrm{KCI}}\,&= \expect\, \left(\sqbrackets{\dotprod{\delta\muac^{\lambda}(\c_i)}{\delta\muac^{\lambda}(\c_j)} + \dotprod{\delta\hatmuac^{\lambda}(\c_i)}{\delta\hatmuac^{\lambda}(\c_j)}} k_{\c}(\c_i, \c_j)\right.\\
     &\qquad\quad\left.\sqbrackets{\dotprod{\delta\mubc^{\lambda}(\c_i)}{\delta\mubc^{\lambda}(\c_j)} + \dotprod{\delta\hatmubc^{\lambda}(\c_i)}{\delta\hatmubc^{\lambda}(\c_j)}}\right)\,,\\
    b_{\mathrm{SplitKCI}}\,&=\expect\, \left(\sqbrackets{\dotprod{\delta\muac^{\lambda}(\c_i)}{\delta\muac^{\lambda}(\c_j)}} k_{\c}(\c_i, \c_j)\right.\\
     &\qquad\quad\left.\sqbrackets{\dotprod{\delta\mubc^{\lambda}(\c_i)}{\delta\mubc^{\lambda}(\c_j)}}\right)\,.\\
\end{align*}

Now, for datasets of equal size, the KCI bias is clearly larger than both SplitKCI ones. However, in practice we would use twice as many points for KCI, so the exact scaling of the terms is important. 

To bound the individual terms, we first note that since we assumed the CME is well-specified, we can use the reproducing property $\dotprod{\muac(\c)}{h}_{\rkhs{\a}}=\dotprod{\muac}{\phi_{\c}(\c)\otimes h}_{\mathrm{HS}}$ to show that $\|\muac(\c)\|\leq \|\muac\| k_{\c}(\c,\c)^{1/2}$, and hence by Cauchy-Schwarz 
\begin{align*}
    \dotprod{\delta\muac^{\lambda}(\c_i)}{\delta\muac^{\lambda}(\c_j)} \leq \|\delta\muac^\lambda\|^2\,k_{\c}(\c_i, \c_i)^{1/2}k_{\c}(\c_j, \c_j)^{1/2}\leq\|\delta\muac^\lambda\|^2\,,
\end{align*}
assuming the kernels are bounded by 1. We can extend this to all products, again using independence of data splits (so  $\|\delta\hatmuac^\lambda\|^2$ and $\|\delta\hatmubc\|^2$ are independent) and the fact that $\|\delta\muac^\lambda\|^2$ and $\|\delta\mubc^\lambda\|^2$ are deterministic,
\begin{align*}
    b_{\mathrm{KCI}}\,&\leq_{\hzero} \sqbrackets{\|\delta\muac^\lambda\|^2 + \expect\,\|\delta\hatmuac^\lambda\|^2}\,\sqbrackets{\|\delta\mubc^\lambda\|^2 + \expect\,\|\delta\hatmubc^\lambda\|^2}\,,\\
    b_{\mathrm{SplitKCI}}\,&\leq_{\hzero} \sqbrackets{\|\delta\muac^\lambda\|^2}\,\sqbrackets{\|\delta\mubc^\lambda\|^2}\,.\\
\end{align*}

Now, with our assumptions about CME convergence, we can directly use Theorem 2 of \cite{li2023optimal}, (with $\gamma=1$, $\alpha=1$ and $\beta\in(1,2]$ to match the required HS norm and the well-specified case): $\|\delta\hatmubc\|^2\leq \tau^2\,K_{\b\c} (m_{\b\c})^{-\frac{\beta-1}{\beta + p}}$ for large enough $m_{\b\c}$ and a positive constant $K_{\b\c}$ with probability at least $1-4e^{-\tau}$ for $\tau\geq 1$. Hence, we can write down the expectation as an integral and split it at $t^*=K_{\b\c} (m_{\b\c})^{-\frac{\beta-1}{\beta + p}}$ corresponding to $\tau=1$:
\begin{align*}
    \expect\,\|\delta\hatmubc\|^2\,&=\int_0^{\infty}\PP{\|\delta\hatmubc\|^2\geq t}dt\\
    &=\int_0^{t^*}\PP{\|\delta\hatmubc\|^2\geq t}dt + \int_{t^*}^{\infty}\PP{\|\delta\hatmubc\|^2\geq t}dt \\
    &\leq \int_0^{t^*}1dt + \int_0^{\infty} 4\,\exp\brackets{-\sqrt{t\,\frac{1}{K_{\b\c}}m_{\b\c}^{\frac{\beta-1}{\beta + p}}}}dt = 9\,K_{\b\c}m_{\b\c}^{-\frac{\beta-1}{\beta + p}}\,.
\end{align*}

By the exact same argument, and using the bias-variance decomposition of the bound (Eq. 12 of \cite{li2023optimal} and then the proof in Appendix A.3),
\begin{align}
    \expect\,\|\delta\hatmuac^\lambda\|^2 \leq \|\delta\muac^\lambda\|^2 + 9\,K_{\a\c}m_{\a\c}^{-\frac{\beta-1}{\beta + p}}.
    \label{app:eq:kci_bias_term}
\end{align}

Finally, by Lemma 1 of A.1 \cite{li2023optimal}, we have a (deterministic) bound (note a different HS norm defined by the $\beta$ interpolation space; the standard norm corresponds to $\beta=1$; see more details in \cite{li2023optimal}):
\begin{equation*}
    \|\delta\muac^\lambda\|^2 \leq \|\muac\|^2_{\beta}\lambda^{\beta-1} = \|\muac\|^2_{\beta}\,m_{\a\c}^{-\frac{\beta-1}{\beta+p}}
\end{equation*}
where we used the optimal learning rate $\lambda=m_{\a\c}^{-\frac{1}{\beta+p}}$. 

It remains to compare the bias in KCI for the full train set and SplitKCI for halved sets. The constant $K_{\a\c}$ in \cref{app:eq:kci_bias_term} is a sum of positive terms, with one of them being $576\,A^2\,\|\muac\|^2_{\beta}$,
where for the well-specified case of $\alpha=1$ and $A^2$ can be defined as the supremum norm of $k_y$ (see \cite{fischer2020sobolev}, Theorem 9). Since we assumed that $k_y$ is bounded by one, the KCI bound (\cref{app:eq:kci_bias_term}) becomes
\begin{gather*}
    \expect\,\|\delta\hatmuac^\lambda\|^2 \leq \|\muac\|^2_{\beta}\,m_{\a\c}^{-\frac{\beta-1}{\beta+p}} + 9\cdot 576\, \|\muac\|^2_{\beta}\,m_{\a\c}^{-\frac{\beta-1}{\beta+p}} + K_1m_{\a\c}^{-\frac{\beta-1}{\beta + p}}\\=(1+K_1)\, \|\muac\|^2_{\beta}\,m_{\a\c}^{-\frac{\beta-1}{\beta+p}} + K_2m_{\a\c}^{-\frac{\beta-1}{\beta + p}}\,,
\end{gather*}
for $K_1 = 5184$ and $K_2 > 0$.

For SplitKCI with half the data, we get
\begin{align*}
    \|\delta\muac^\lambda\|^2 \leq \|\muac\|^2_{\beta}\,\brackets{\frac{m_{\a\c}}{2}}^{-\frac{\beta-1}{\beta+p}} \leq 2\|\muac\|^2_{\beta}\,m_{\a\c}^{-\frac{\beta-1}{\beta+p}}\,,\\
    \|\delta\mubc^\lambda\|^2 \leq \|\mubc\|^2_{\beta}\,\brackets{\frac{m_{\b\c}}{2}}^{-\frac{\beta-1}{\beta+p}} \leq 2\|\mubc\|^2_{\beta}\,m_{\b\c}^{-\frac{\beta-1}{\beta+p}}\,.
\end{align*}

Combining the bounds and denoting $K_3=K_{\b\c}$  for simplicity finishes the proof.
\end{proof}

\subsection{Wild bootstrap} \label{app:proofs:wild-bootstrap}
\wild*

Before beginning the proof, we note that in practice we use the unbiased estimator in \cref{app:eq:hsic_ub}, which corresponds to a U-statistic.
The only difference between the U- and V-statistics of Equations \ref{eq:kci_estimator} and \ref{app:eq:hsic_ub} is essentially removal of the terms $i=j$ in the sum (the scaling is also adapted accordingly). See \citet[Lemma 22]{kim2023differentially} for an expression of the difference between U- and V-statistics for the closely-related HSIC case.

\begin{proof}
We start by setting up the notations. 
Let
\begin{gather*}
    (K^c_{\A})_{ij}\!=\!\dotprod{\phi_{\a}(\A_i)\!-\!\muac(\C_i)}{\phi_{\a}(\A_j)\!-\!\muac(\C_j)},\\(\widehat{K}^c_{\A})_{ij}\!=\!\dotprod{\phi_{\a}(\A_i)\!-\!\hatmuac(\C_i)}{\phi_{\a}(\A_j)\!-\!\hatmuac(\C_j)},
\end{gather*}
and similarly for $\B$ instead of $\A$. 
The $V$-statistics and wild bootstrap KCI estimates using the true and estimated conditional mean embeddings are given as 
\begin{align*}
    &V = \frac{1}{n^2}\sum_{1\leq i,j \leq n}(K^c_{\A})_{ij} (K_{\C})_{ij} (K^c_{\B})_{ij}, & &V^* = \frac{1}{n^2}\sum_{1\leq i,j \leq n}  q_i q_j\, (K^c_{\A})_{ij}  (K_{\C})_{ij} (K^c_{\B})_{ij}, \\
    &\widehat V = \frac{1}{n^2}\sum_{1\leq i,j \leq n}(\widehat K^c_{\A})_{ij} (K_{\C})_{ij} (\widehat K^c_{\B})_{ij}, & &\widehat V^* = \frac{1}{n^2}\sum_{1\leq i,j \leq n}q_i q_j\, (\widehat K^c_{\A})_{ij}  (K_{\C})_{ij} (\widehat K^c_{\B})_{ij}.
\end{align*}
We focus on the KCI case, as the proof for CIRCE follows very similarly (the only difference is that the centered kernel matrix $K^c_{\A}$ is replaced with its non-centered version $K_{\A}$).
The reasoning also applies to SplitKCI.
We emphasize that the conditional mean embeddings  $\hatmuac$, $\hatmubc$ do not admit close forms and have to be estimated using the CME estimator (see Equation \ref{eq:cme}) on \emph{held-out} data.

We prove below that 
\begin{enumerate}
\item[(i)] under the null  $n(\widehat V - V)\to 0$ and $n(\widehat V^* -  V^*)\to 0$ in probability,
\item[(ii)] under the alternative $(\widehat V - V)\to 0$ and $(\widehat V^* -  V^*)\to 0$ in probability.
\end{enumerate}
Before proving (i) and (ii), we first show that they imply the desired results. 
\citet{chwialkowski2014wild} study the wild bootstrap generally and handle the MMD and HSIC as special cases. 
Similarly to their HSIC setting, a symmetrisation trick can be performed for KCI and CIRCE.
We rely on their general results which assume a weaker $\tau$-mixing assumption which is trivially satisfied the i.i.d. setting we consider, we point that similar results can be found in \citet{wu1986jackknife,shao2010dependent,leucht2013dependent}.

Under the null, we have
$$
n(\widehat V - \widehat V^*) = n(\widehat V - V) + n(V - V^*) + n(V^* - \widehat V^*)
$$
where by (i) $n(\widehat V - V)\to 0$ and $n(V^* - \widehat V^*)\to 0$ in probability. 
\citet[Theorem 1]{chwialkowski2014wild} also guarantees that $n(V - V^*)\to 0$ weakly. 
We conclude that $n(\widehat V - \widehat V^*)\to0$ weakly, that is, $n\widehat V^*$ and $n\widehat V$ converge weakly to the same distribution under the null.

Under the alternative, by (ii) we have $(\widehat V - V)\to 0$ and $(V^* - \widehat V^*)\to 0$ in probability.
Moreover, \citet[Theorems 3 and 2]{chwialkowski2014wild} guarantees that $V^*\to 0$ in mean squared and that $V\to c$ in mean squared for some $c>0$, respectively. We conclude that
\begin{equation*}
\widehat V^* = (\widehat V^* - V^*) + V^* \to 0  \qquad \textrm{ and }\qquad
\widehat V = (\widehat V - V) + V \to c
\end{equation*}
in probability. 
\end{proof}

\begin{proof}[Proof of (i) and (ii)]
This proof adapts the reasoning of \citet[Lemma 3]{chwialkowski2014wild} for HSIC to hold for KCI. This is particularly challenging because while the mean embedding admits a closed form, the conditional mean embedding does not and has weaker convergence properties \citep{li2023optimal,li2023towards}. We provide results in probability rather than in expectation (one can use Markov's inequality to turn an expected bound into one holding with high probability).
Let us define\footnote{Note that the scaling is different from \citet[Equation 7]{chwialkowski2014wild}.}
\begin{align*}
    T_n &= \frac{1}{n} \sum_{i=1}^n Q_i (\phi_{\a}(\A_i) - \muac(\C_i)) \otimes \phi_{\c}(\C_i) \otimes (\phi_{\b}(\B_i) - \mubc(\C_i)), \\
    \widehat T_n &= \frac{1}{n} \sum_{i=1}^n Q_i (\phi_{\a}(\A_i) - \hatmuac(\C_i)) \otimes \phi_{\c}(\C_i) \otimes (\phi_{\b}(\B_i) - \hatmubc(\C_i)),
\end{align*}
where, either $Q_i = q_i$, $i=1,\dots,n$ Rademacher variable in which case 
$$
\|T_n\|^2 = \frac{1}{n^2}\sum_{1\leq i,j \leq n}  q_i q_j\, (K^c_{\A})_{ij}  (K_{\C})_{ij} (K^c_{\B})_{ij} = V^* 
 \qquad \textrm{ and }\qquad
\|\widehat T_n\|^2 = \widehat V^*,
$$
or $Q_i = 1$, $i=1,\dots,n$ in which case $\|T_n\|^2 = V$ and $\|\widehat T_n\|^2 = \widehat V$. So it suffices to prove
\begin{enumerate}
\item[(i)] under the null  $n(\|\widehat T_n\|^2 - \|T_n\|^2)\to 0$ in probability,
\item[(ii)] under the alternative $(\|\widehat T_n\|^2 - \|T_n\|^2)\to 0$ in probability.
\end{enumerate}

We can obtain these results provided that 
\begin{enumerate}
\item[(iii)] $\|\sqrt{n} (\widehat T_n - T_n)\| \to 0$ in probability,
\end{enumerate}
which we prove below. For now, assume (iii) holds.

Under the null, we then have
\begin{align*}
n \,\Big|\|\widehat T_n\|^2-\| T_n\|^2\Big| 
&\leq n\, \Big|\|\widehat T_n\|-\| T_n\|\Big| \, \Big|\|\widehat T_n\|+\| T_n\|\Big| \\
&\leq \sqrt{n}\,\|(\widehat T_n - T_n)\|  \Big(\sqrt{n}\,\|\widehat T_n\|+\sqrt{n}\,\| T_n\|\Big) \\
&\to 0
\end{align*}
in probability, 
provided that $\sqrt{n}\| T_n\|<\infty$, which is guaranteed by \citet[Lemma 4]{chwialkowski2014wild} under the null ($\Delta=0$ in Lemma 4 notation since we're working with the i.i.d. case), and that $\sqrt{n}\,\|\widehat T_n\|<\infty$, both with arbitrarily high probability. The latter holds since
$$
\sqrt{n}\,\|\widehat T_n\| 
\leq \sqrt{n}\,\|\widehat T_n - T_n\| + \sqrt{n}\,\| T_n\|
$$
where the first term tends to zero due to the assumed $m_{\a\c},m_{\b\c}$ scaling and the second is bounded, with arbitrarily high probability.
This proves (i).

Under the alternative, we obtain
\begin{align*}
\Big|\|\widehat T_n\|^2-\| T_n\|^2\Big| 
&\leq \Big|\|\widehat T_n\|-\| T_n\|\Big| \, \Big|\|\widehat T_n\|+\| T_n\|\Big| \\
&\leq \|(\widehat T_n - T_n)\|  \Big(\|\widehat T_n\|+\| T_n\|\Big) \\
&\to 0
\end{align*}
in probability, as $\| T_n\|<\infty$ by bounding the kernels and the $Q_i$'s by 1,
and
$$
\|\widehat T_n\| 
\leq \sqrt{n}\,\|\widehat T_n - T_n\| + \| T_n\|
$$
with again the first term tending to zero and the second one being bounded, with arbitrarily high probability.
This proves (ii).
\end{proof}

\begin{proof}[Proof of (iii)]
\citet[Theorem 2.2.a]{li2023towards}, with $\gamma=1$, $\beta\in(1,2]$, $\alpha=1$ (see \cref{assumpt:cme}), provides a rate of convergence in RKHS norm for the CME estimator when the true conditional mean embedding lies within a $\beta$-smooth subset of the RKHS ($\beta=1$ would correspond to the full RKHS).
There exists a constant $C>0$ independent of $m_{\a\c}$ and $\tau$ such that
$$
\|\hatmuac - \muac\| \leq C \tau m_{\a\c}^{-\frac{\beta-1}{2(\beta+p)}} \qquad \textrm{with probability at least } 1-5e^{-\tau}
$$
for $m_{\a\c}$ sufficiently large and for all $\tau>\ln(5)$,
where $p\in(0,1]$ induces a smoothness condition on the kernel used (its eigenvalues $\mu_i$ must decay as  $i^{-1/p}$).
As justified below, we will be interested in the scaled rate
$$
\|\sqrt{n}(\hatmuac - \muac)\|
\leq C \tau n^{\frac{1}{2}}m_{\a\c}^{-\frac{\beta-1}{2(\beta+p)}}.
$$
Now, assuming that $m_{\a\c} =\omega \Big(n^{\frac{\beta+p}{\beta-1}}\Big)$ (\emph{i.e.} $m_{\a\c}$ dominates $n^{\frac{\beta+p}{\beta-1}}$ asymptotically), then we obtain 
\begin{equation}
\label{eq:cme_rate_1}
\|\sqrt{n}(\hatmuac - \muac)\|
\to 0
\end{equation}
in probability as $n$ tends to infinity.
Note that this trivially implies that 
\begin{equation}
\label{eq:cme_rate_2}
\|\hatmuac\| 
\leq 
\|\sqrt{n}(\hatmuac - \muac)\|
+ \| \muac\| 
<\infty
\end{equation}
for all $n\in\mathbb{N}$.
The results in Equations (\ref{eq:cme_rate_1} and \ref{eq:cme_rate_2} also hold using $\B$ instead of $\A$.

We are now ready to show that $\|\sqrt{n} (\widehat T_n - T_n)\| \to 0$ in probability.
Note that 
\begin{align}
\sqrt{n} (\widehat T_n - T_n)
&= \frac{1}{\sqrt{n}} \sum_{i=1}^n Q_i \Bigg( (\phi_{\a}(\A_i) - \hatmuac(\C_i)) \otimes \phi_{\c}(\C_i) \otimes (\phi_{\b}(\B_i) - \hatmubc(\C_i)) \nonumber\\
&\hspace{2.4cm} -(\phi_{\a}(\A_i) - \muac(\C_i)) \otimes \phi_{\c}(\C_i) \otimes (\phi_{\b}(\B_i) - \mubc(\C_i))
\Bigg)\nonumber\\
&= \frac{1}{\sqrt{n}} \sum_{i=1}^n Q_i \,
    \phi_{\a}(\A_i) \otimes \phi_{\c}(\C_i) \otimes \Big(\mubc(\C_i)
    - \hatmubc(\C_i)
\Big) \label{eq:difference_term_1}\\
&\hspace{0.5cm}+ \frac{1}{\sqrt{n}} \sum_{i=1}^n Q_i \,
    \Big(\muac(\C_i) - \hatmuac(\C_i)\Big) \otimes \phi_{\c}(\C_i) \otimes \phi_{\b}(\B_i) \label{eq:difference_term_2}\\
&\hspace{-0.5cm}+ \frac{1}{\sqrt{n}} \sum_{i=1}^n Q_i \, 
    \Big(\hatmuac(\C_i) \otimes \phi_{\c}(\C_i) \otimes\hatmubc(\C_i) - \muac(\C_i) \otimes\phi_{\c}(\C_i) \otimes \mubc(\C_i)\Big). \label{eq:difference_term_3}
\end{align}
The terms (\ref{eq:difference_term_1}) and (\ref{eq:difference_term_2}) can be handled in a similar fashion, the term (\ref{eq:difference_term_3}) needs to be tackle separately.
Note that in the CIRCE case we would have only the first term (\ref{eq:difference_term_1}).

The term (\ref{eq:difference_term_1}) is 
\begin{align*}
&\left\|\frac{1}{\sqrt{n}} \sum_{i=1}^n Q_i \,
    \phi_{\a}(\A_i) \otimes \phi_{\c}(\C_i) \otimes \Big(\mubc(\C_i)
    - \hatmubc(\C_i)
\Big) \right\| \\
\leq\ & \frac{1}{n} \sum_{i=1}^n |Q_i|\left\| \,
    \phi_{\a}(\A_i) \otimes \phi_{\c}(\C_i)  \otimes\sqrt{n}\Big(\mubc(\C_i)
    - \hatmubc(\C_i)
\Big) \right\| \\
=\ & \frac{1}{n}\sum_{i=1}^n (k_{\a}(\A_i, \A_i)k_{\c}(\C_i, \C_i))^{1/2} \left\| \sqrt{n}\Big(\mubc(\C_i)
    - \hatmubc(\C_i)
\Big) \right\| \\
\leq\ & \frac{1}{n}\sum_{i=1}^n k_{\a}(\A_i, \A_i)^{1/2} k_{\c}(\C_i, \C_i) \left\| \sqrt{n}\Big(\mubc
    - \hatmubc
\Big) \right\|_{\mathrm{HS}} \\
\leq\ & \left\| \sqrt{n}\Big(\mubc- \hatmubc
\Big) \right\| \\
\to\ &0
\end{align*}
in probability by Equation (\ref{eq:cme_rate_1}), where we the triangle inequality (2nd line), product space definition (3rd line), $\|\muac(\c)\|\leq \|\muac\| k_{\c}(\c,\c)^{1/2}$ (4th line) and the facts that $k_{\a}$, $k_{\c}$ and $Q_i$ are all bounded by 1.

Similarly, the second term  (\ref{eq:difference_term_2}) converges to 0 in probability.

For the third term (\ref{eq:difference_term_3}), we have 
\begin{align*}
&\left\|    \frac{1}{\sqrt{n}} \sum_{i=1}^n Q_i \, \Big(\hatmuac(\C_i) \otimes \phi_{\c}(\C_i)\otimes\hatmubc(\C_i) - \muac(\C_i) \otimes \phi_{\c}(\C_i)\otimes\mubc(\C_i)\Big) \right\| \\
\leq\ &\frac{1}{n}\sum_{i=1}^n|Q_i|\left\|    \sqrt{n} \, \Big(\hatmuac(\C_i) \otimes \phi_{\c}(\C_i)\otimes\hatmubc(\C_i) - \muac(\C_i) \otimes \phi_{\c}(\C_i)\otimes\mubc(\C_i)\Big) \right\| \\
=\ &\frac{1}{n}\sum_{i=1}^n \sqrt{k_{\c}(\C_i, \C_i)}\norm{\sqrt{n} \, \brackets{\hatmuac(\C_i) \otimes\brackets{\hatmubc(\C_i)\pm \mubc(\C_i)} - \muac(\C_i) \otimes\mubc(\C_i)}} \\
\leq\ &\frac{1}{n}\sum_{i=1}^n \sqrt{k_{\c}(\C_i, \C_i)}\left(\norm{\sqrt{n} \, \hatmuac(\C_i) \otimes\brackets{\hatmubc(\C_i)- \mubc(\C_i)}} +\right.\\
&\qquad\qquad\qquad\left.\norm{\sqrt{n} \, (\hatmuac(\C_i)- \muac(\C_i)) \otimes\mubc(\C_i)}\right) \\
=\ &\frac{1}{n}\sum_{i=1}^n \sqrt{k_{\c}(\C_i, \C_i)}\left(\norm{\hatmuac(\C_i)}\norm{\sqrt{n} \, \brackets{\hatmubc(\C_i)- \mubc(\C_i)}} +\right.\\
&\qquad\qquad\qquad\left.\norm{\mubc(\C_i)}\norm{\sqrt{n} \, (\hatmuac(\C_i)- \muac(\C_i))}\right) \\
\leq\ &\frac{1}{n}\sum_{i=1}^n \brackets{k_{\c}(\C_i, \C_i)}^{3/2}\brackets{\norm{\hatmuac}\norm{\sqrt{n} \, \brackets{\hatmubc- \mubc}} + \norm{\mubc}\norm{\sqrt{n} \, (\hatmuac- \muac}} \\
\leq\ &\Big\|\hatmuac\Big\|
\Big\|\sqrt{n}\, (\hatmubc - \mubc)\Big\|
+
\Big\|\mubc\Big\|
\Big\|\sqrt{n}\,(\hatmuac-\muac)\Big\| \\
\to\ &0
\end{align*}
in probability by Equation (\ref{eq:cme_rate_1}), using the triangle inequality (second line), product spaces definition (third line), again triangle inequality (4th) and product spaces definition (5th) and then $\|\muac(\c)\|\leq \|\muac\| k_{\c}(\c,\c)^{1/2}$ and the fact that $k_{\c}(\c,\c)$ and $Q_i$ are bounded by 1, with the final result holding since $\big\|\hatmuac\big\|<\infty$ by Equation (\ref{eq:cme_rate_2}).

We deduct that $\|\sqrt{n} (\widehat T_n - T_n)\| \to 0$ in probability, which shows (iii).

This concludes the proof of \cref{th:wild}.
\end{proof}

\section{Measures of conditional independence}
\label{app:sec:other_methods}
\subsection{Kernel ridge regression}
For completeness, we include the kernel ridge regression algorithm with leave-one-out error (for RKHS-valued outputs, see \citealp{pogodin2022efficient}) in \cref{alg:krr}. While it is possible to optimize the LOO errors via gradient descent over $\lambda$ and kernel parameters, in practice this can be computationally demanding. In our experiments, we choose the parameters over a pre-defined grid.

\begin{figure}[ht]
\centering
\input{paper_content/algorithms/krr_alg}
\end{figure}
\subsection{KCI}
\label{app:sec:kci}
\paragraph{Notes on the original implementation}
First, we note that the set $E_1'=\{g\in L_{\A}^2 : \expect\sqbrackets{g(\A)\given \C}=0 \}$ used in the original KCI formulation \citep{zhang2012kernel} is not present in Theorem 1 of \citet{daudin1980partial}. The theorem is proven for $E_1=\{g\in L_{\A\C}^2 : \expect\sqbrackets{g(\A,\C)\given \C}=0 \}$, and stated for $g\in L^2_{\A\C}$ and $g\in L^2_{\A}$. However, we note that $E_1'$ functions can be constructed as $g=g' - \expect\sqbrackets{g'\given \C}$ for $g'\in L_{\A}^2$. Therefore, given $h\in \{h\in L^2_{\B\C} : \expect\sqbrackets{h(\B,\C)\given \C}=0\}$, we have $\EE{gh}=0$ iff $\EE{g'h}=0$ since $\EE{\expect\sqbrackets{g'\given \C}h(\B,\C)}=\EE{\expect\sqbrackets{g'\given \C}\expect\sqbrackets{h(\B,\C)\given \C}}=0$. Therefore, the proof for this case follows the $g\in L_{\A}^2$ case of \cite{daudin1980partial} (see Corollary A.3 of \citealt{pogodin2022efficient} for the required proof).

Second, the original KCI definition estimated the $\phi_{\b\c}(\B,\C) - \mu_{\B\C\given \C}(\C)$ residual. For radial basis kernels, this can be decomposed into $\phi_{\c}(\C)\otimes\brackets{\phi_{\b} - \mu_{\B\given \C}}$, avoiding estimation of $\mu_{\C\given\C}=\phi_{\c}(\C)$ (see \citealp{pogodin2022efficient}). This is crucial to ensure consistency of the regression, because estimating $\mu_{\C\given\C}$ is equivalent to estimating the identity operator --- for a characteristic RKHS $\rkhs{\c}$, this operator is not Hilbert-Schmidt and thus the regression problem is not well-specified. This is described in Appendix B.9 of \citet{mastouri2021proximal} (also see Appendix D, $Y=X$ case, of \citealt{li2023optimal}). In practice, computing the $\C\!\rightarrow\!\C$ regression can lead to significant problems, as described by \citet[Appendix B.9 and Figure 4]{mastouri2021proximal}. For 1d Gaussian $\C$, the CME estimator correctly estimates the identity on the high density regions of the training data; in the tail of the distribution, however, the estimate becomes heavily biased since those points are rarely present in the training data. 

\paragraph{Gamma approximation for KCI}
The Gamma approximation of $p$-values was used by \cite{zhang2012kernel} (and previously suggested for independence testing with HSIC by \citealt{gretton2007kernel}). This approximation uses the fact that under the null, the kernel-based statistics are distributed as a weighted (infinite) sum of $\chi^2$ variables $\sum \lambda_i \xi_i^2$ for $\xi_i\sim\NN(0, 1)$. While this sum is intractable, it can be approximated as a Gamma distribution $p(t)=t^{k-1}\frac{e^{-t/\theta}}{\theta^k\gamma(k)}$ with parameters (for the KCI estimator $\frac{1}{n^2}\tr(KL)$ with centered matrices $K$ and $L$) estimated as
\begin{align*}
    \mu \,&= \frac{1}{n}\sum_{i=1}^n K_{ii}L_{ii}\,,\quad \sigma^2 = 2\frac{1}{n^2}\sum_{i,j}K_{ij}^2L_{ij}^2\,,\qquad k = \frac{\mu}{\sigma^2}\,,\quad\theta = \frac{\sigma^2}{n\,\mu}\,.
\end{align*}
The Gamma approximation is not exact, and higher-order moment matching methods can improve approximation quality \citep{bodenham2016comparison}.
In our tests with KCI and its variants, the Gamma approximation produced $p$-values that are consistently lower than the required level under $\hzero$ (\cref{fig:task1_gamma}, left for wild bootstrap vs. right for Gamma approximation). This might be attributed not to the approximation itself (e.g. it performed well in \citep{gretton2007kernel,zhang2012kernel}) but rather to large CME estimation errors in our tasks that introduce bias in mean/variance estimation for the method. For CIRCE, the situation was reversed, which could also be potentially explained by (the lack of) CME estimation bias in one of its terms.

\begin{figure}[ht]
    \centering
    \includegraphics[width=\textwidth]{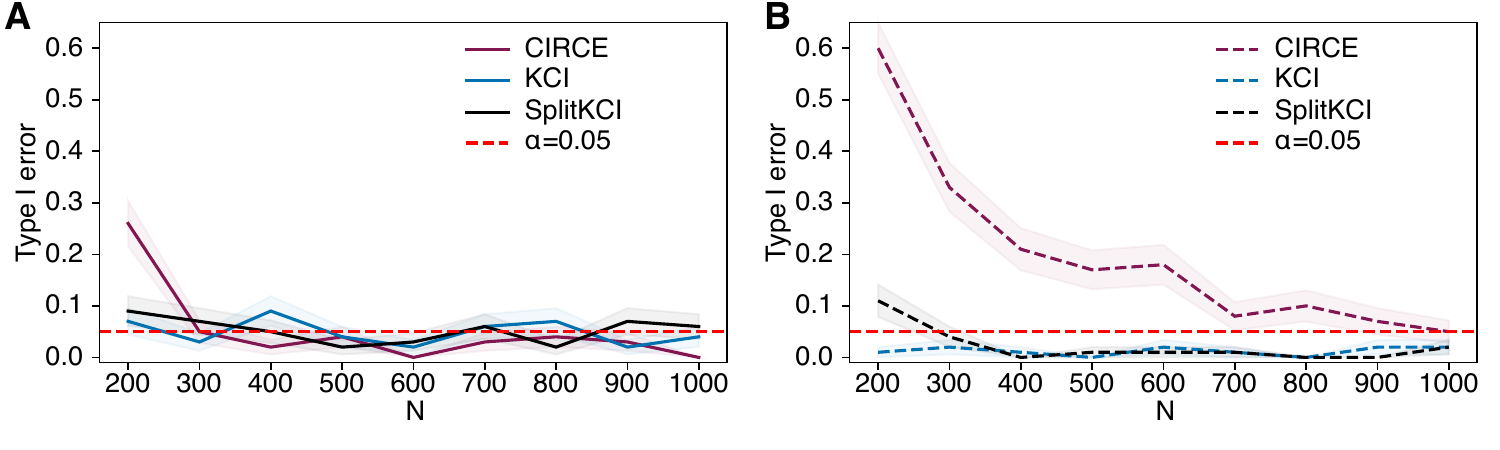}
    \caption{Type I error comparison between and \textbf{A.} wild bootstrap and \textbf{B.} the Gamma approximation methods of $p$-value estimation for Gaussian data (see \cref{app:toy_tasks}). Lines/shaded area: mean/$\pm$SE over 100 trials, $\alpha=0.05$.}
    \label{fig:task1_gamma}
\end{figure}

\begin{figure}[ht]
\centering
\input{paper_content/algorithms/train_test_alg}
\end{figure}

For SplitKCI, we use the train/test splitting procedure outlined in \cref{subsec:ttsplit}. See \cref{alg:train_test_split} for a more detailed description. Note that this can be viewed as testing with the CIRCE operator \eqref{eq:circe}, in which we test $\C\indep\A\given\C$,
where the kernel over the conditioning variable is just $k_\c(\c, \c') = 1$.

\subsection{GCM}
\label{app:sec:gcm}
    The Generalised Covariance Measure (GCM, \citealp{shah2020hardness}) relies on estimates of the conditional mean
    \begin{equation*}
        \widehat f(\C) \approx \EE{\A\given \C}\,,\quad \widehat g(\C) \approx \EE{\B\given \C}\,.
    \end{equation*}

    For $n$ samples $\a_i\in \RR^{d_{\a}}$ and $\b_i\in \RR^{d_{\b}}$, define
    \begin{align*}
        R_{ikl} \,&= (\a_i - f(\c_i)))_k (\b_i - g(\c_i))_l\,,\quad \bar R_{kl} = \frac{1}{n}\sum_{i=1}^n R_{ikl}\,,\quad T_{kl} = \frac{\sqrt{n}\bar R_{kl}}{\sqrt{\frac{1}{n}\sum_{j=1}^n (R_j)_{kl}^2 - (\bar R_{kl})^2}}\,.
    \end{align*}
    
    For 1-dimensional $\A$ and $\B$ (and arbitrary $\C$), the $p$-value is computed as $p=2\,(1-\Phi(|T_{11}|))$ for the standard normal CDF $\Phi$.

    For higher-dimensional $\A$ and $\B$ (with $d_{\a} d_{\b}\geq 3$), the statistic becomes $S=\max_{k,l}|T_{kl}|$, and the $p$-value is computed by approximating the null distribution of $T_{kl}$ as a zero-mean Gaussian with the covariance matrix (for a $n$-dimensional vector $R^{kl}$ corresponding to all data points and dimensions $k,l$)
    \begin{equation*}
        \Sigma_{kl,pd}=\frac{\frac{1}{n}\sum_{i=1}^n R_{ikl}R_{ipd} - \bar R_{kl}\bar R_{pd}}{\sqrt{\frac{1}{n}\sum_{i=1}^n R_{ikl}^2 - \bar R_{kl}^2} \sqrt{\frac{1}{n}\sum_{i=1}^n R_{ipd}^2 - \bar R_{pd}^2}}\,.
    \end{equation*}
\subsection{RBPT2}
\label{app:sec:rbpt2}

The Rao-Blackwellized Predictor Test (\citealp{polo2023conditional}; ``2'' stands for the test variant that doesn't require approximation of $P(\B\given \C)$) first trains a predictor $g(\B,\C)$ of $\A$, e.g. by computing $g(\B,\C)\approx \EE{\A\given \B,\C}$. (The original paper used a linear regression; we test this setup and also a CME with a Gaussian kernel.) Then, it builds a second predictor through CME as $h(\C) = \EE{g(\B,\C)\given \C}$ (the original paper used a polynomial kernel of degree two; we obtain similar results for a Gaussian kernel for linear $g$ and a set of Gaussian/polynomial kernels for a Gaussian $g$). Then, for a convex loss $l$, the test statistic $S$ is computed as
\begin{equation*}
    T_i = l(h(\c_i), \a_i) - l(g(\b_i,\c_i), \a_i)\,,\quad  S= \frac{\sqrt{n}\sum_{i=1}^{n} T_i}{\sqrt{\sum_{i=1}^{n} T_i^2 - (\sum_{i=1}^{n} T_i)^2}}\,.
\end{equation*}
The $p$-value is then computed as $p=1-\Phi(S)$ (one-sided test); the basic intuition is that if $\B$ has no extra information about $\A$ apart from the joint $\C$-dependence, the second predictor $h$ will perform as well as the first one $g$.

\begin{figure}[ht]
    \centering
    \includegraphics[width=0.5\textwidth]{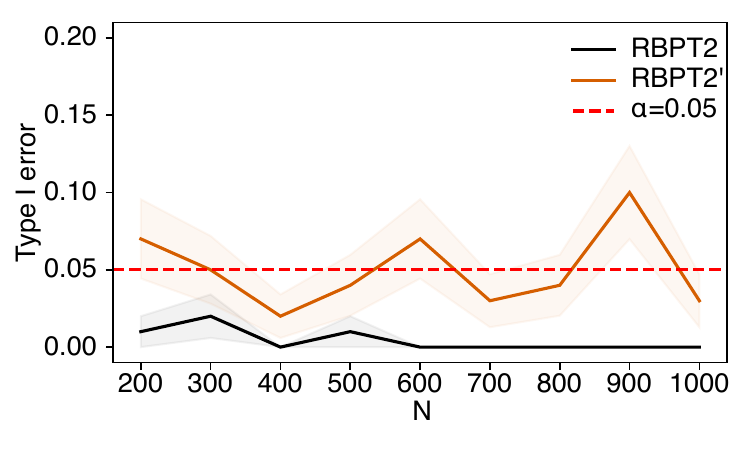}
    \caption{Type I error comparison between RBPT2 (black) vs its unbiased version RBPT2' (orange) for Gaussian data (see \cref{app:toy_tasks}). Lines/shaded area: mean/$\pm$SE over 100 trials, $\alpha=0.05$.}
    \label{fig:task1_rbpt_bias}
\end{figure}

For \citet{polo2023conditional}, the loss (for one-dimensional $\A,\B$) was the mean squared error loss $l(g, \a)=(g-\a)^2$. For $d_{\a}$-dimensional $\a$, we also used the MSE loss $l(g, \a)=\|g-\a\|_2^2$ and found the test to consistently produce much lower Type I error than the nominal level $\alpha$. We show an example of this behaviour for Gaussian data in \cref{fig:task1_rbpt_bias} (black line); for the RatInABox task in the same setup as  in \cref{fig:rat_standard}, the test would always reject regardless of the ground truth (not shown).
This indicates that the statistic $S$ doesn't follow a standard normal distribution. We argue that this happens due to the CME estimation error, and propose a debiased version of RBPT2 (which computes correct $p$-values; \cref{fig:task1_rbpt_bias}, orange line):
\begin{theorem}
    Denote the approximation error of $\EE{\A\given \B,\C}$ via $g(\B,\C)$ as $\delta g$:
    \begin{equation*}
        g(\B,\C) = \EE{\A\given \B,\C} + \delta g(\B,\C)\,.
    \end{equation*}
    Assume that the second regression $h$ uses much more data (i.e. the unbalanced data regime in the main text) and so correctly estimates $h=\EE{g(\B,\C)\given \C}$, and that RBPT2 uses the MSE loss $l$.

    Then, under $\hzero$, the unnormalized test statistic $T(\A,\B,\C)=\|h(\C)-\A\|_2^2-\|g(\B,\C)-\A\|_2^2$ has a negative bias:
    \begin{equation*}
        \EE{T(\A,\B,\C)} =_{\hzero} -\EE{\|g(\B,\C)- h(\C)\|_2^2}.
    \end{equation*}
\end{theorem}
\begin{proof}
    Under $\hzero$, we have that 
    \begin{align*}
        g(\B,\C) \,&= \mu_{\A\given \B\C}(\B,\C) + \delta g(\B,\C) =_{\hzero} \mu_{\A\given \C}(\C) + \delta g(\B,\C)\,,\\
        \quad h(\C) \,&= \mu_{\A\given \C}(\C) + \EE{\delta g(\B,\C)\given \C} = \mu_{\A\given \C}(\C) + \mu_{\delta g}(\C)\,. 
    \end{align*}
    We can therefore decompose the loss difference as
    \begin{align*}
        &\|h(\C) - \A\|_2^2 - \|g(\B,\C) - \A\|_2^2  =_{\hzero}  \|\mu_{\A\given \C}(\C)- \A + \mu_{\delta g}(\C) \|_2^2 - \|\mu_{\A\given \C}(\C)- \A + \delta g(\B,\C)\|_2^2\\
        &\qquad=2\,(\mu_{\A\given \C}(\C)- \A)\T \mu_{\delta g}(\C) + \|\mu_{\delta g}(\C)\|_2^2 - 2\,(\mu_{\A\given \C}(\C)- \A)\T \delta g(\B,\C) - \|\delta g(\B,\C)\|_2^2\,.
    \end{align*}
    We can take the expectation of this difference under $\hzero$, zeroing out the cross-product terms since $\A\indep \B\given \C$:
    \begin{gather*}
        \EE{\|h(\C) - \A\|_2^2 - \|g(\B,\C) - \A\|_2^2}  =_{\hzero} -\EE{\|\delta g(\B,\C)\|_2^2 - \|\mu_{\delta g}(\C)\|_2^2}\\
        = -\EE{\|\delta g(\B,\C)- \mu_{\delta g}(\C)\|_2^2} = -\EE{\|g(\B,\C)- h(\C)\|_2^2}\,.
    \end{gather*}
\end{proof}
Conveniently, we can estimate this bias from the $g,h$ estimates, resulting in a bias-corrected RBPT2 (called RBPT2' in the main text) with a test statistic
\begin{equation*}
    T_i = \|h(\c_i) - \a_i\|_2^2 - \|g(\b_i,\c_i) - \a_i\|_2^2 + \|g(\b_i,\c_i) - h(\c_i)\|_2^2\,.
\end{equation*}
While the derivation holds under the null, we note that adding a positive term to $T_i$ should decrease Type II errors since we're testing for large positive deviations in $T_i$.

\section{Experimental details and additional experiments}
\label{app:experiments}
In all cases, we ran wild bootstrap 1000 times to estimate $p$-values. For kernel ridge regression, the $\lambda$ values were chosen (via leave-one-out) from $[10\,\delta, \dots,10^7\delta]$ (log scale) for SVD tolerance $\delta=\|K_{\C}\|_2\,\eps$  and machine precision $\eps$ (for \texttt{float32} values). 

For the experiments with the post-nonlinear model (\cref{subseq:postnonlin_res}) and the additional synthetic task (\cref{app:toy_tasks}) with KCI, SplitKCI, and CIRCE, the kernels over $\A$ and $\B$ were Gaussian $k(x,x')=\exp(-\|x-x'\|_2^2/\sigma^2)$ (note that $\sigma^2$ is not multiplied by two in our implementation) with $\sigma^2=1$ as the data was normalized. The kernel over $\C$ for $\mu_{\B\given\C}$ was always Gaussian, with $\sigma^2$ chosen (via leave-one-out) from [0.1, 0.2, 0.5, 1.0, 1.5, 2.0]. For GCM, we use the same setup for $\widehat f(\C) \approx \EE{\A\given \C},\, \widehat g(\C) \approx \EE{\B\given \C}$ (see the definitions in \cref{app:sec:gcm}). For RBPT2, we used the linear kernel for $g(\B,\C)$ and the Gaussian kernel for $h(\C)$ with the same parameter choice as for $\hatmuac$ (see the definitions in \cref{app:sec:rbpt2}). 

For the synthetic rat data experiments and KCI-style methods, the kernels over $\A$ and $\B$ remained the same. For $\C$, we parametrized the Gaussian kernel as $k(x,x')=\exp(-\sum_{i=1}^{d}(x_i-x'_i)^2\gamma^2_i/d)$. For $\hatmubc$, we chose $\gamma_i=\gamma$ (the same for all $i$) from [0.1, 0.2, 0.5, 1.0, 1.5, 2.0]. For $\hatmuac$, we additionally chose from $\gamma_1=\gamma_2=0$ and $\gamma_3=\gamma_4=0$ options (no head direction/no position coordinates) with $d=2$ in those cases. As $\A$ only depends on head direction, such features selection improves regression error. The same approach was used for GCM.
For RBPT, we used the Gaussian kernel with $\sigma^2=1$ for $g$; for $h$, we used the same kernel as for $\hatmuac$. This was done to reflect the increased model flexibility in other measures; in the original experiments of \cite{polo2023conditional}, $g$ was found through (non-ridge) regression and $h$ used a polynomial kernel of degree 2. We found this setup to be similar to our first setup in our preliminary experiments (not shown).

For all synthetic data experiments, we used 100 random seeds, plotting plotted mean $\pm$SE for the error rates. The train/test splitting procedure for SplitKCI (\cref{alg:train_test_split}) used $\alpha=0.5$ and test ratio $\beta$ ranging from 0.1 to 0.5, such that $\beta N \geq \max(100, 0.1\,N)$ with a step of $\beta N$ of 50.

For the car insurance data, we followed exactly the same procedure as \cite{polo2023conditional} for RBPT, RBPT2, RBPT2' and GCM. For kernel-based measures and the simulated $\hzero$, the kernel over $\A$ was Gaussian with $\sigma^2$ determined from sample variance over train data; the kernel over $\B$ was a simple $k_{\b}(\b,\b')=\ind[\b=\b']$ since $\B$ is binary; the kernel over $\C$ was Gaussian with $\sigma^2$ chosen from $[0.5, 1.0, 2.0]\cdot \widehat\sigma^2$ for train sample variance $\widehat\sigma^2$. For the actual $p$-values, we split the train data in two randomly for independent $\mu_{\A\given \C}$ and $\mu_{\B\given \C}$ regression (to avoid the high memory requirements of the full dataset). The rejection indicators $p_{ij}$ for seed $i$ for the simulated $\hzero$ task were computed for each company $j$ individually. As in \cite{polo2023conditional}, we did not make the variance correction through the law of total variance, essentially estimating average rejection rate across companies, since the companies are fixed and not sampled randomly. Therefore, the standard errors are correct (for this quantity). 

For KCI, CIRCE and SplitKCI with the Gamma approximation, we used the biased HSIC-like estimator in \cref{eq:kci_estimator} (note that the Gamma approximation is derived for the biased one). For wild bootstrap, we used the unbiased HSIC estimator \citep{song2007supervised} (the biased one performed marginally worse; not shown) for kernel matrices $K,L$:
\begin{equation}
    \frac{1}{m(m-3)}\brackets{\tr\brackets{KL} + \frac{1\T K1 1\T L 1}{(m-1)(m-2)} - \frac{2}{m-2}1^T KL 1}\,.
    \label{app:eq:hsic_ub}
\end{equation}

\subsection{Additional toy task}
\label{app:toy_tasks}

We sample  a circular $\C$ with $\A$ and $\B$ being noisy version of $\C$, with conditional dependence along the first coordinate under $\hone$. For $\hzero$, 
\begin{gather*}
    \C=\frac{\xi_1}{\|\xi_1\|_2}\,,\quad\A = \C + \gamma\,\xi_2+\gamma\,\xi_\a\,,\quad\B = \C + \gamma\,\xi_3+\gamma\,\xi_\b\,.
\end{gather*}

For $\hone$, part of the noise in $\A$ and $\B$ is shared through $\xi_{\a\b}$,
\begin{gather*}
    \C=\frac{\xi_1}{\|\xi_1\|_2}\,,\quad\A = \C + \gamma\,\xi_2+\gamma\,\begin{bmatrix}
        \xi_{\a\b} \\ \xi_{\a\a}
    \end{bmatrix}\,,\quad\B = \C + \gamma\,\xi_3+\gamma\,\begin{bmatrix}
        \xi_{\a\b} \\ \xi_{\b\b}
    \end{bmatrix}\,,
\end{gather*}
where $\xi_1,\xi_2,\xi_3,\xi_\a,\xi_\b\sim\NN(0, I_2)$, $\xi_{\a\b},\xi_{\a\a},\xi_{\b\b}\sim\NN(0, 1)$ and 
for $\gamma=0.05$.

For both \cref{fig:task1_gamma,fig:task1_rbpt_bias}, we varied the total budget from $N=200$ to $N=1000$, always using 100 points for testing and the rest for training.

\subsection{Post-nonlinear model}
\label{app:sec:kci_task}

In addition to the main results in \cref{subseq:postnonlin_ttsplit,subseq:postnonlin_res}, we evaluated all methods (including CIRCE) without train/test splitting (\cref{app:fig:kci_task_no_split}) and with a fixed split of $n=100$ test points (\cref{app:fig:kci_task_split}). In all cases, CIRCE showed trivial results (almost perfect Type I error but also Type II error close to 1, meaning that it would mostly just reject the alternative). RBPT' suffered from not having a train/test split, while KCI and SplitKCI, in contrast, worked well even without a split on this task.

\begin{figure}[ht]
     \centering
     \includegraphics[width=0.95\textwidth]{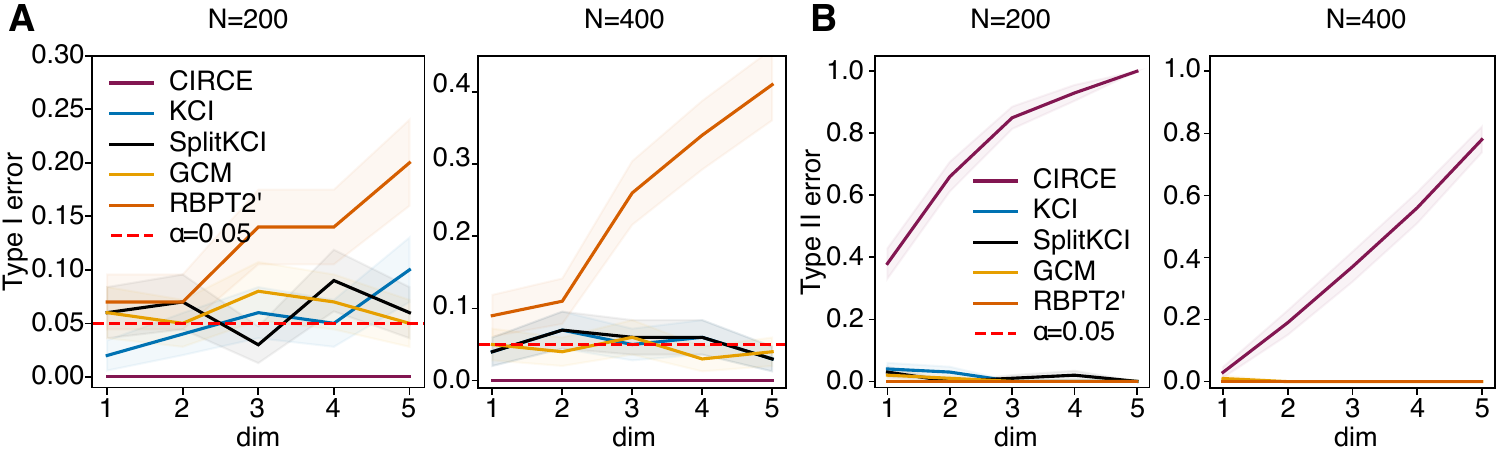}
     \caption{Post-nonlinear model experiments for increasing dimensionality of the task. Unlike in \cref{fig:kci_task_best}, all methods use all data points for both testing and training. \textbf{A.} Type I error for $N=200$ (left) and $N=400$ (right) data points. \textbf{B.} Type II error for $N=200$ (left) and $N=400$ (right) data points. Lines/shaded area: mean/$\pm$SE over 100 trials, $\alpha=0.05$.}
     \label{app:fig:kci_task_no_split}
\end{figure}

\begin{figure}[ht]
     \centering
     \includegraphics[width=0.95\textwidth]{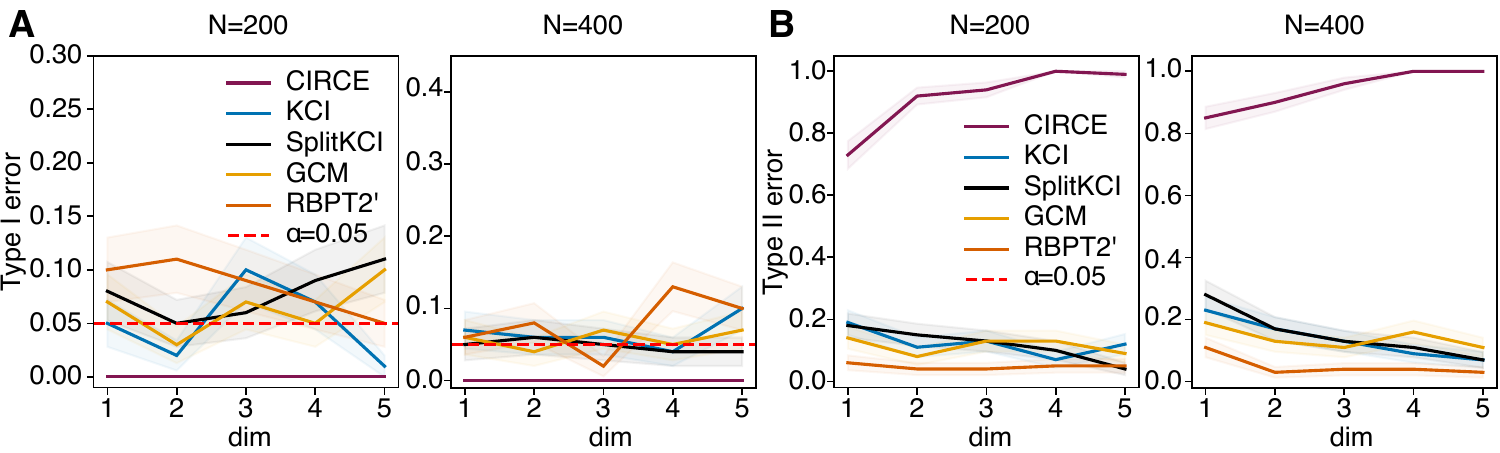}
     \caption{Post-nonlinear model experiments for increasing dimensionality of the task. Unlike in \cref{fig:kci_task_best}, all methods use 100 data points for testing and the rest for training. \textbf{A.} Type I error for $N=200$ (left) and $N=400$ (right) data points. \textbf{B.} Type II error for $N=200$ (left) and $N=400$ (right) data points. Lines/shaded area: mean/$\pm$SE over 100 trials, $\alpha=0.05$.}
     \label{app:fig:kci_task_split}
\end{figure}

\subsection{Synthetic neural data}
\label{app:sec:synthetic}

\begin{figure}[ht]
     \centering
     \includegraphics[width=0.95\textwidth]{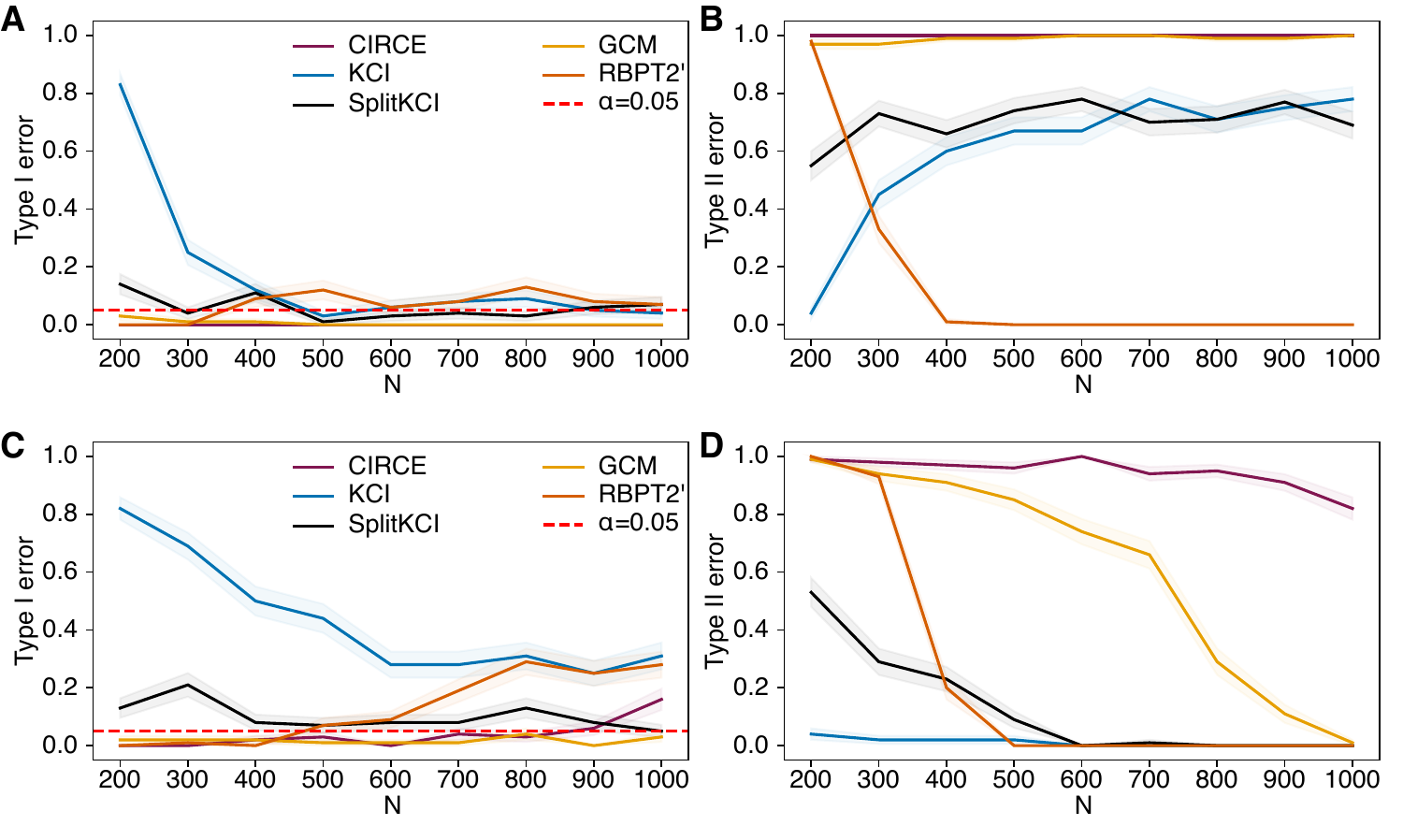}
     \caption{Synthetic neural data experiments for increasing $(A,B,C)$ dataset size. \textbf{A.} Type I error for using 100 points (out of $N$) as test data. \textbf{B.} Type II error. \textbf{C-D.} Same as \textbf{A-B} but for $N/2$ points as test data. Lines/shaded area: mean/$\pm$SE over 100 trials, $\alpha=0.05$.}
     \label{app:fig:rat_std_100_05}
\end{figure}

\begin{figure}[ht]
     \centering
     \includegraphics[width=0.95\textwidth]{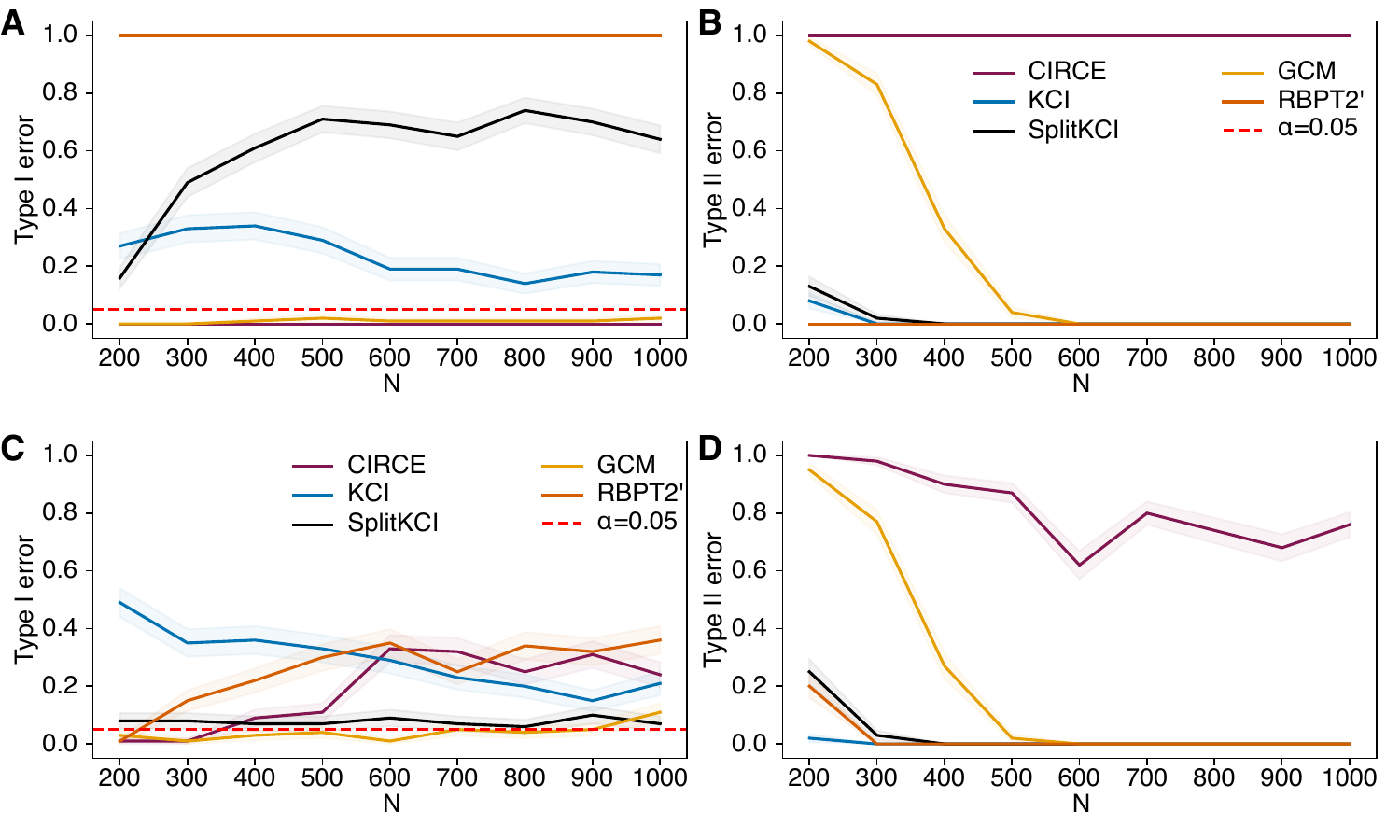}
     \caption{Synthetic neural data experiments for increasing $(A,B,C)$ dataset size. \textbf{A.} Type I error without a train/test split (i.e. using the same $N$ points for each part). \textbf{B.} Type II error. \textbf{C-D.} Same as \textbf{A-B} but for an independent set of $N$ test points (control experiment). Lines/shaded area: mean/$\pm$SE over 100 trials, $\alpha=0.05$.}
     \label{app:fig:rat_std_no_split_control}
\end{figure}

\begin{figure}[ht]
     \centering
     \includegraphics[width=0.95\textwidth]{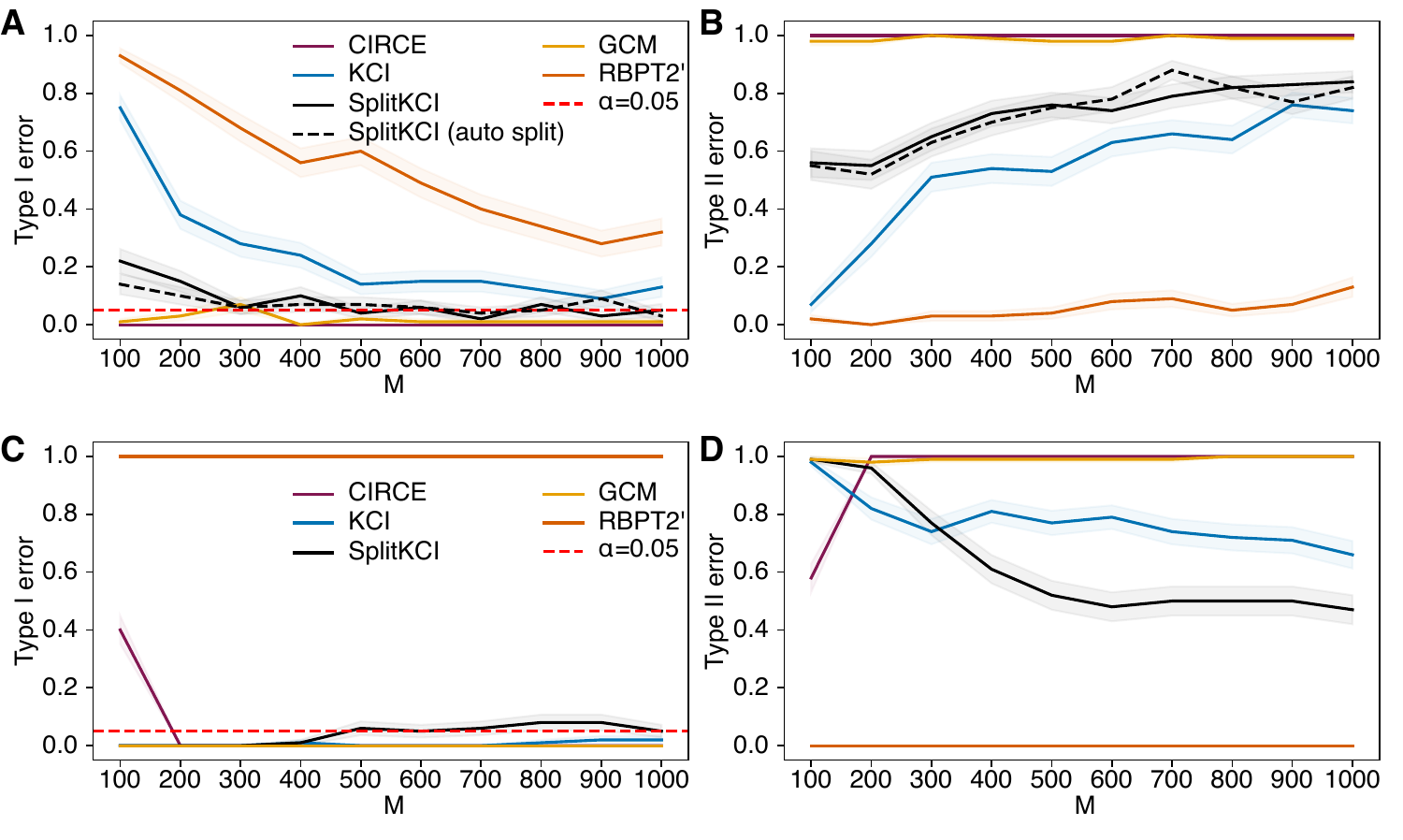}
     \caption{Synthetic neural data experiments for 200 $(A,B,C)$ dataset size, but increasing amount of auxiliary $(B,C)$ data. \textbf{A.} Type I error for 100/100 train/test split. \textbf{B.} Type II error. \textbf{C-D.} Same as \textbf{A-B}, but without a train/test split. Lines/shaded area: mean/$\pm$SE over 100 trials, $\alpha=0.05$.}
     \label{app:fig:rat_aux_200}
\end{figure}

\begin{figure}[ht]
     \centering
     \includegraphics[width=0.95\textwidth]{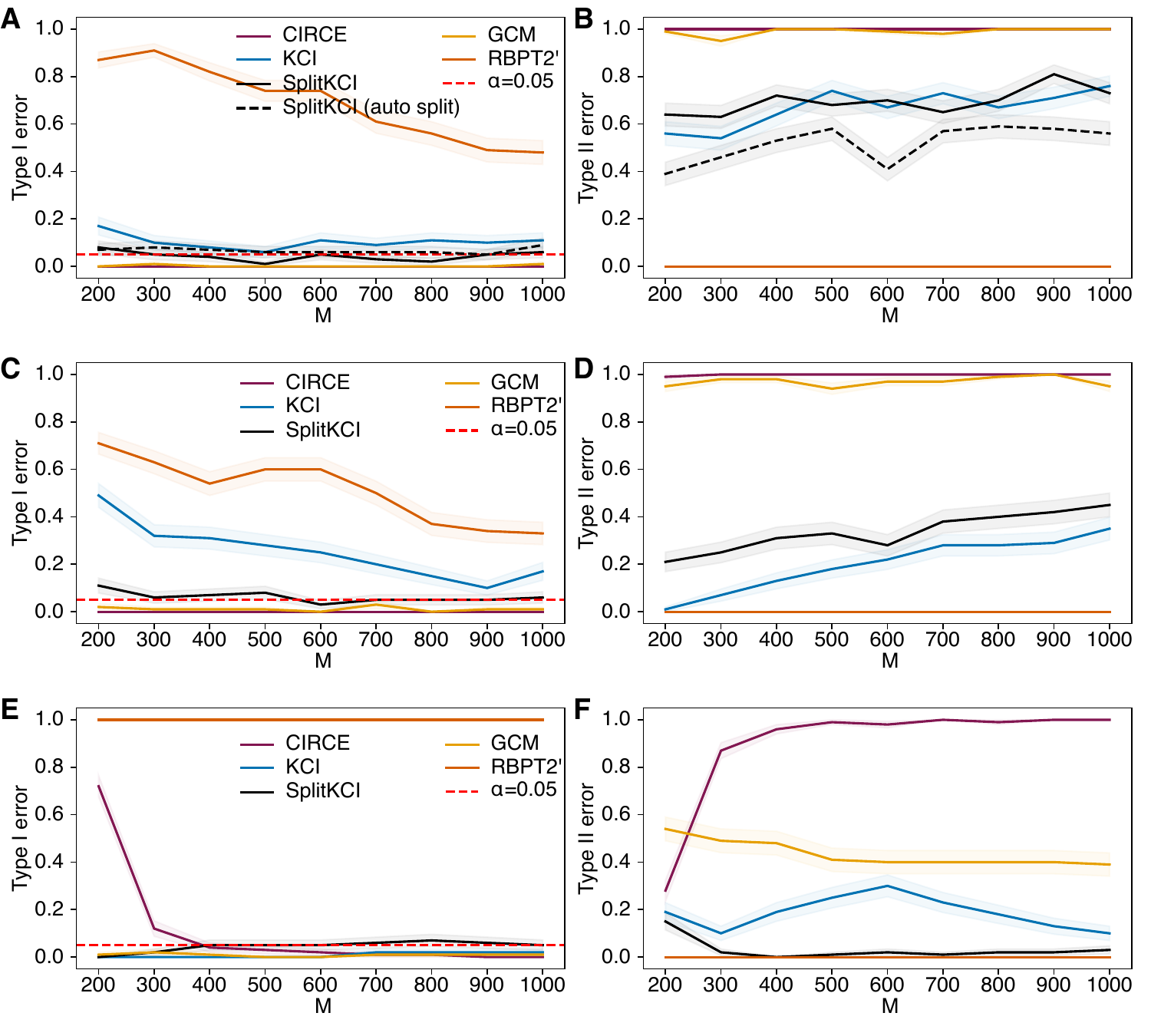}
     \caption{Synthetic neural data experiments for 400 $(A,B,C)$ dataset size, but increasing amount of auxiliary $(B,C)$ data. \textbf{A.} Type I error for 300/100 train/test split. \textbf{B.} Type II error. \textbf{C-D.} Same as \textbf{A-B}, but for 200/200 train/test split. \textbf{E-F.} Same as \textbf{A-B}, but without a train/test split. Lines/shaded area: mean/$\pm$SE over 100 trials, $\alpha=0.05$.}
     \label{app:fig:rat_aux_400}
\end{figure}

To generate the data, we used the \href{https://github.com/RatInABox-Lab/RatInABox}{RatInABox toolbox} (MIT license) \cite{George2024}. We used a rectangular box with four walls blocking the central part of the box, so the simulated rat runs in circles around the center. This adds a marginal dependence between head direction and position. 

We then simulated 100 head directional cells $\A$ and 100 grid cells $\B$ with their underlying firing rate $r_{\a}(t),r_{\b}(t)$ bounded between 0 and 1 Hz, an Ornstein–Uhlenbeck (OU) process $\xi_{\a}(t),\xi_{\b}(t)$ with 0.1 std and a fast time constant. The observed firing rate for head directional cells was $\widehat r_{\a}(t)=\mathrm{ReLU(r_{\a}(t) + \xi_{\a}(t))}$. Under the null, we constructed conjunctive cells from grid cells $\B$ and an independent set of head directional cells $\A'$ as $\widehat r_{\b}(t)=\mathrm{ReLU(\widehat r_{\a'}(t) + r_{\b}(t) + \xi_{\b}(t))}$. Under the alternative, we constructed the conjunctive cells as $\widehat r_{\b}(t)=\mathrm{ReLU(\widehat r_{\a}(t) + r_{\b}(t) + \xi_{\b}(t) - 1)}$ ($-1$ makes sure the cells are active only when both the head direction and position are correct). We simulated the rat for several minutes for each random seed, subsampling the points so the autocorrelation of the OU process is negligible to get approximately i.i.d. data. The exact data generation process is provided in the code.

In the main text, \cref{fig:rat_standard,fig:rat_aux} show the best performing settings w.r.t. data splits for each presented algorithm, with SplitKCI using the \cref{alg:train_test_split} to determine the split ratio in the standard data regime. We also evaluated all methods for fixed data splits (\cref{app:fig:rat_std_100_05,app:fig:rat_std_no_split_control,app:fig:rat_aux_200,app:fig:rat_aux_400}). Omitted in the main text, CIRCE showed reasonable Type I performance along with very poor Type II performance (purle), and SplitKCI without splitting over $\mubc$ (Split KCI($\A$ only), green, in the figures) performed very closely to SplitKCI (however, it was sensitive to train/test split ratio and incompatible with our splitting heuristic). 

\subsubsection{Results without train/test splitting}\label{app:subsubseq:control}
\changed{
We also evaluated the methods without the train/test splits, i.e. using all data for both training and testing. The results are shown in \cref{app:fig:rat_std_no_split_control}A-B. As expected, GCM holds level and quickly gains power as the number of points $N$ increases. CIRCE shows pathological performance as it always accepts the null. The other tests struggle to hold level.

To confirm why this is happening, we repeated the same experiment, but now resampling the test sets; see \cref{app:fig:rat_std_no_split_control}C-D. That is, the number of train/test points was the same as in panels A-B, but the test set was independent of the training one. KCI did not improve performance, showing that the balance of test and train points is indeed important. Interestingly, SplitKCI could hold level for an independent set of points. This could be explained by additional correlations for in-sample errors: if the CME estimation for a given training point is poor for both regressions (i.e. $\C\rightarrow \A$ and $\C\rightarrow \B$), it will add a large error term to the test statistic (which is more likely for SplitKCI as it's using fewer points than KCI). For an independent test set, such errors are more likely to be uncorrelated.
}

\subsection{Car insurance data}
\label{app:sec:real_data}

For the car insurance data (full dataset, i.e. merged for all companies), we compute the $p$-values as in \cite{polo2023conditional}. The train/test split of the dataset is done as 70/30\%.  For KCI and SplitKCI, we further split the train set in half for each $\mu_{\A\given\C}$ and $\mu_{\B\given\C}$ regression due to increased memory requirements compared to other methods. 

\clearpage

\vskip 0.2in
\bibliography{bibliography}

\end{document}